\RequirePackage{fix-cm}
\documentclass[smallcondensed]{svjour3}     
\smartqed  

\usepackage{graphicx}

\usepackage{amssymb,amsfonts,amsmath,array}
\usepackage{thmtools,thm-restate}

\usepackage[square,sort,comma,numbers]{natbib}
\usepackage{url}
\usepackage{fixfoot}
\usepackage{scrextend}

\usepackage{algorithm}
\usepackage{algorithmicx}
\usepackage{algpseudocode}
\usepackage{enumerate}

\usepackage{color,soul}
\usepackage{subfig}

\usepackage{booktabs}
\usepackage{multirow}
\usepackage{siunitx}
\usepackage{multicol}
\usepackage{array}
\usepackage[table,xcdraw]{xcolor}
\usepackage[figuresright]{rotating}

\DeclareMathOperator*{\argmax}{argmax}
\DeclareMathOperator*{\argmin}{argmin}

\algrenewcomment[1]{//\textit{#1}}

\newcommand{\algrule}[1][.2pt]{\par\vskip.5\baselineskip\hrule height #1\par\vskip.5\baselineskip}

\newcommand{\ind}{\perp}

\newcommand{\condind}[3]{(#1 {\ind} #2 \ | \ #3)}

\newcommand{\norm}[1]{\left\lVert#1\right\rVert}
\begin{document}

\title{Massively-Parallel Feature Selection for Big Data}

\author{
Ioannis Tsamardinos\footnote{\label{first}Equal contribution.}\and Giorgos Borboudakis\footref{first} \and Pavlos Katsogridakis \and Polyvios Pratikakis \and Vassilis Christophides        
}

\authorrunning{Tsamardinos, I., Borboudakis, G. et al.} 

\institute{
Ioannis Tsamardinos \at \email{tsamard.it@gmail.com} \and
Giorgos Borboudakis \at \email{borbudak@gmail.com} \and
Pavlos Katsogridakis \at \email{pkatsogr@gmail.com} \and
Polyvios Pratikakis \at \email{polyvios@ics.forth.gr} \and
Vassilis Christophides \at \email{christop@csd.uoc.gr}
}

\date{Received: date / Accepted: date}

\maketitle


\begin{abstract}
We present the {\em Parallel, Forward-Backward with Pruning} (PFBP) algorithm for \emph{feature selection} (FS) in Big Data settings (high dimensionality and/or sample size). To tackle the challenges of Big Data FS PFBP partitions the data matrix both in terms of rows (samples, training examples) as well as columns (features). By employing the concepts of $p$-values of conditional independence tests and meta-analysis techniques PFBP manages to rely only on computations local to a partition while minimizing communication costs. Then, it employs  powerful and safe (asymptotically sound) heuristics to make early, approximate decisions, such as {\em Early Dropping} of features from consideration in subsequent iterations, {\em Early Stopping} of consideration of features within the same iteration, or {\em Early Return} of the winner in each iteration. PFBP provides asymptotic guarantees of optimality for data distributions {\em faithfully} representable by a causal network (Bayesian network or maximal ancestral graph). Our empirical analysis confirms a super-linear speedup of the algorithm with increasing sample size, linear scalability with respect to the number of features and processing cores, while dominating other competitive algorithms in its class. 

\end{abstract}

\section{Introduction}
    
Creating predictive models from data requires sophisticated machine learning, pattern recognition, and statistical modeling techniques. When applied to Big Data settings these algorithms need to scale not only to millions of training instances (samples) but also millions of predictive quantities (interchangeably called features, variables, or attributes) \cite{Zhao2013,Zhai2014,Boln-Canedo2015}. A common way to {\em reduce data dimensionality} consists of selecting only a subset of the original
features that retains all of the predictive information regarding an outcome of interest $T$. Specifically, the objective of the Feature Selection (FS) problem can be defined as identifying a feature subset that is of {\em minimal-size} and {\em collectively} (multivariately) {\em optimally predictive}\footnote{Optimally predictive with respect to an ideal predictor; see \cite{Tsamardinos2003} for a discussion.} w.r.t. $T$\footnote{This definition covers what we call single FS; the problem of multiple FS can be defined as the problem of identifying all minimal and optimally-predictive subsets but it has received much less study in the literature \cite{Statnikov2013,MXM16}.}. By removing {\em irrelevant as well as redundant} (related to the concept of {\em weakly relevant}) features, FS essentially facilitates the learning task. It results in predictive models with fewer features that are easier to inspect, visualize, understand, and faster to apply. Thus, FS provides valuable intuition on the data generating mechanism and {\em is a primary tool for knowledge discovery; deep connections of the solutions to the FS with the causal mechanisms that generate the data have been found} \cite{Tsamardinos2003}. Indeed, FS is often the primary task of an analysis, while predictive modeling is only a by-product.

Designing a FS algorithm is challenging because by definition it is a combinatorial problem; the FS is NP-hard even for linear regression problems \cite{Welch1982}. An exhaustive search of all feature subsets is impractical except for quite small sized feature spaces. Heuristic search strategies and approximating assumptions are required to scale up FS, ranging from convex relaxations and parametric assumptions such as linearity (e.g., the Lasso algorithm \cite{Tibshirani1996}) to causally-inspired, non-parametric assumptions, such as  faithfulness of the data distribution to a causal model \cite{Pearl1991,Spirtes2000}. 

Specifically, in the context of Big Data featuring both high dimensionality and/or high sample volume, computations become {\em CPU-} and {\em data-intensive} that cannot be handled by a single machine\footnote{See \cite{Zhao2013,Zhai2014,Boln-Canedo2015} for the evolution of Big Data dimensionality in various ML datasets.}. The main challenges arising in this context are (a) how can data be partitioned both horizontally (over samples) and vertically (over features), called \emph{hybrid-partitioning}, so that \emph{computations can be performed locally in each block} and combined globally with a \emph{minimal communication overhead}; (b) what \emph{heuristics} can quickly (e.g., without the need to go through all samples) and safely (providing theoretical guarantees of correctness) eliminate irrelevant and redundant features. Hybrid partitioning over both data samples and learned models \cite{XING2016179, DBLP:conf/nips/LeeKZHGX14} is an open research issue in Big ML algorithms while safe FS heuristics has been proposed only for sparse Big Data \cite{Singh2009,Gallego17}, i.e., for data where a large percentage of values are the same (typically zeros).  

To address these challenges we introduce the {\bf Parallel, Forward-Backward with Pruning (PFBP)} algorithm for Big Volume Data. PFBP does not relies on data sparsity and is generally applicable to both dense and sparse datasets; in the future, it could be extended to include optimizations specifically designed for sparse datasets. PFBP is based on \emph{statistical tests of conditional independence} and it is inspired by statistical causal modeling that represents join probability distribution as a causal model and specifically the theory of Bayesian networks and maximal ancestral graphs \cite{Pearl2000,Spirtes2000,SpirtesRichardson2002}. 

To tackle parallelization with hybrid partitioning (challenge (a) above), PFBP decisions rely on $p$-values and log-likelihoods returned by the independence tests computed {\em locally} on each data partition; these values are then combined together using {\em statistical meta-analysis techniques} to produce {\em global} approximate $p$-values\footnote{Alternatively, one can combine the test statistics that produce the $p$-values. This is conceptually equivalent, although there may be differences in practice. 
} and log-likelihoods. This technique essentially minimizes PFBP's communication cost, as only local $p$-values and log-likelihoods need to be communicated from workers to the master node in a  cluster of machines at each iteration of the algorithm.

To reduce the number and workload of iterations required to compute a FS solution (challenge (b) above), PFBP relies on several heuristics. First, it adapts for Big Data a heuristic called {\em Early Dropping} recently introduced in \cite{Borboudakis2017}. Early Dropping removes features from subsequent iterations thus significantly speeding up the algorithm. Then, PFBP is equipped with two new heuristics for \emph{Early Stopping} of consideration of features within the same iteration, and \emph{Early Returning} the current best feature for addition or removal. The three heuristics are implemented using \emph{Bootstrap-based statistical tests}. They are applied on the set of currently available local $p$-values and log-likelihoods to determine whether the algorithm has seen enough samples to make safely (i.e., with high probability) early decisions. 

PFBP is proven to compute the optimal feature set for distributions {\em faithful} \cite{Spirtes2000} (also called stable distributions \cite{Pearl1991}) to a causal network represented as a Bayesian Network or a maximal ancestral graph \cite{Spirtes2000, SpirtesRichardson2002}. These are data distributions whose set of conditional independencies coincides with the set of independencies entailed by a causal graph and the Markov Condition \cite{Pearl1991, Spirtes2000}. Assuming faithfulness of the data distribution has led to algorithms that have been proven competitive in practice \cite{Margaritis2000, Aliferis2003HITON, Tsamardinos2003MMPC, Tsamardinos2003IAMB, Pena2007, Aliferis2010JMLR, Lagani2010, Lagani2013, MXM16,Borboudakis2017}. We should also notice that all PFBP computations are not bound to specific data-types; by supplying different conditional independence tests PFBP becomes applicable to a wide variety of data types and target variables \cite{MXM16} (continuous, ordinal, nominal, time-to-event).

The paper is organized as follows. In Section~\ref{sec:back} we provide a brief introduction to the basic concepts required to introduce our FS algorithm. The PFBP algorithm is introduced in Section~\ref{sec:alg}. In Section~\ref{sec:pruning} we explain the heuristics used by PFBP in detail, and show to how to implement them using bootstrap-based tests. Guidelines for setting the hyper-parameter values for the data partitioning used by PFBP are presented in Section~\ref{sec:parameters}. In Section~\ref{sec:pfbp:practical} we list some implementation details of PFBP, which are required for a fast and robust implementation. The theoretical properties of PFBP are presented in Section~\ref{sec:theory}. A high-level theoretical comparison of PFBP to alternative feature selection algorithms, as well as an overview of feature selection methods for Big Data is given in Section~\ref{sec:related}. Finally, in Section~\ref{sec:experiments} we evaluate PFBP on synthetic data, and compare it to alternative forward-selection algorithms on 11 binary classification datasets.

\section{Background and Preliminaries}
\label{sec:back}

\begin{table}[!t]
\centering
	\caption{Table containing common acronyms, terms and mathematical notation (left) used throughout the paper with a short description (right).}
    \label{tbl:notation}
  \begin{tabular}{ll}
    \toprule
    FBS & Forward-Backward Selection\\
    PFBP & Parallel Forward-Backward with Pruning\\
    UFS & Univariate Feature Selection\\
    SFO & Single Feature Optimization\\
    ED & Early Dropping\\
    ES & Early Stopping\\
    ER & Early Return\\
    \toprule
    Iteration & forward (backward) iteration of PFBP\\
    Phase & forward (backward) loop of PFBP\\
    Run & execution of a forward and a backward Phase by PFBP\\
    Feature Subset & subset of features\\
    Sample Subset & subset of samples\\
    Data Block & contains samples of one Sample Subset and one Feature Subset\\
    Group Sample & set of Sample Subsets\\
    Group & set of Data Blocks corresponding to Sample Subsets in a Group Sample\\    
    \toprule
    $X$ & Random variable\\
    $\mathbf{X}$ & set of random variables\\
    $|\mathbf{X}|$ & number of elements in $\mathbf{X}$\\
    $T$ & outcome (or target) variable\\
    $Test(X_k,T|\mathbf{S})$ & conditional independence test of $X_k$ with $T$ given $\mathbf{S}$ \\
    $p_k$ & $p$-value of $Test(X_k,T|\mathbf{S})$ (for some $\mathbf{S}$) \\
    df & Degrees of Freedom\\
    $\alpha$ & significance level threshold \\
    $\mathcal{D}$ & Dataset - 2-D matrix\\
    $\mathbf{F}$ & Features in $\mathcal{D}$\\
    $F_j$ & $j$-th Feature Subset\\
    $nf$ & number of Feature Subsets\\
    $f$ & number of features in each Feature Subset\\
    $S_i$ & $i$-th Sample Subset\\
    $ns$ & number of Sample Subsets\\
    $s$ & number of samples in each Sample Subset\\
    $G_q$ & $q$-th Group Sample\\
    $Q$ & number of Group Samples\\
    $C$ & number of Sample Subsets per Group Sample\\
    $D_{i,j}$ & Data Block with rows $S_i$ and columns $F_j$\\
    $\Pi$ & 2-D matrix with local log $p$-values\\
    $\pi_{i,j}$ & local $p$-value of $j$-th alive variable in $\Pi$ computed on rows in $S_i$\\
    $\pi$ & Vector with combined log $p$-values\\
    $\pi_i$ & combined log $p$-value for the $i$-th alive variable\\
    $\mathbf{S}$ & set of Selected features\\
    $\mathbf{R}$ & set of Remaining features\\
    $\mathbf{A}$ & set of Alive features\\
    $B$ & number of bootstrap iterations used by bootstrap tests\\
    $^b$ & value corresponding to $b$-th bootstrap sample\\
    $P_{drop}$ & threshold used by bootstrap test for Early Dropping\\
    $P_{stop}$ & threshold used by bootstrap test for Early Stopping\\
    $P_{return}$ & threshold used by bootstrap test for Early Return\\
    $lt$ & tolerance level used by bootstrap test for Early Return\\
    \bottomrule
  \end{tabular}
\end{table}

In this section, we provide the basic notation used throughout the paper, and present the core algorithmic and statistical reasoning techniques exploited by the proposed FS algorithm. Random variables are denoted using upper-case letters (e.g. X), while sets of random variables are denoted using bold upper-case letters (e.g. $\mathbf{Z}$). We use $|\mathbf{Z}|$ to refer to the number of variables in $\mathbf{Z}$. The terms variable and feature will be used interchangeably, and the outcome (or target) variable will be denoted as $T$. A summary of acronyms, terms and notation is given in Table~\ref{tbl:notation}.

\subsection{Forward-Backward Feature Selection}
\label{sec:back:fbs}
The Forward-Backward Selection algorithm (FBS) is an instance of the stepwise feature selection algorithm family \citep{Kutner2004, Weisberg2005}. It is also one of the first and most popular algorithms for causal feature selection \cite{Margaritis2000,Tsamardinos2003IAMB}. In each forward \textbf{Iteration}, FBS selects the feature that provides the largest increase in terms of predictive performance for $T$, and adds it to the set of selected variables, denoted with $\mathbf{S}$ hereon, starting from the empty set. The forward \textbf{Phase} ends when no feature further improves performance or a maximum number of selected features has been reached. In each Iteration of the backward Phase, the feature that most confidently does not reduce performance is removed from $\mathbf{S}$. The backward Phase stops when no feature can be removed without reducing performance. We use the terms \textbf{Phase} to refer to the forward and backward loops of the algorithm and \textbf{Iteration} to the part that decides which feature to add or remove next.

\begin{algorithm}[!t]
\caption{Forward-Backward Selection}
\label{alg:fbs}
\begin{algorithmic}
	\Require Dataset $\mathcal{D}$, Target $T$, Significance Level $\alpha$
	\Ensure Selected Features $\mathbf{S}$
    \State $\mathbf{S} \gets \emptyset$ \Comment{\textit{Selected features, initially empty}}
    \State \Comment{\textit{Forward Phase: Iterate until no more features can be selected}}
	\While{$\mathbf{S}$ changes} 
	\State \Comment{\textit{Identify $V^*$ with minimum $p$-value conditional on $\mathbf{S}$}}
	\State $V^*$ $\gets$ $\argmin\limits_{V \in \mathbf{V}_{\mathcal{D}} \setminus \mathbf{S}}$ \textproc{Pvalue}$(T,V|\mathbf{S})$
	\State \Comment{\textit{Select $V^*$ if conditionally dependent with $T$ given $\mathbf{S}$}}
	\If{\textproc{Pvalue}$(T,V^*|\mathbf{S})$ $\leq$ $\alpha$}
	\State $\mathbf{S}$ $\gets$ $\mathbf{S} \cup V^*$
	\EndIf
	\EndWhile
    \State \Comment{\textit{Backward Phase: Iterate until no more features can be removed from $\mathbf{S}$}}
	\While{$\mathbf{S}$ changes} 
	\State \Comment{\textit{Identify $V^*$ with maximum $p$-value conditional on $\mathbf{S} \setminus V^*$}}
	\State $V^*$ $\gets$ $\argmax\limits_{V \in \mathbf{S}}$ \textproc{Pvalue}$(T,V|\mathbf{S} \setminus V)$
	\State \Comment{\textit{Remove $V^*$ if conditionally independent with $T$ given $\mathbf{S} \setminus V^*$}}
	\If{\textproc{Pvalue}$(T,V^*|\mathbf{S} \setminus V^*$) $>$ $\alpha$}
	\State $\mathbf{S}$ $\gets$ $\mathbf{S} \setminus V^*$
	\EndIf
	\EndWhile	
	\State {\bfseries return} $\mathbf{S}$
\end{algorithmic}
\end{algorithm}

To determine whether predictive performance is increased or decreased when a single feature is added or removed in a greedy fashion, FBS uses conditional independence tests\footnote{Alternatively, one can use information criteria such as AIC \citep{Akaike1973} and BIC \citep{Schwarz1978}, or out-of-sample methods such as cross-validation to evaluate the performance of the current set of selected features; see \citep{Kutner2004, Weisberg2005} for more details.}. 
An important advantage of methods relying on conditional independence tests is that it \textit{allows one to adapt and apply the algorithm to any type of outcome (e.g. nominal, ordinal, continuous, time-to-event, time-course, time series) for which an appropriate statistical test of conditional independence exists}. This way, the same feature selection algorithm can deal with different data types\footnote{For example, the R-package MXM \cite{MXM16} includes asymptotic, permutation-based, and robust tests for nominal, ordinal, continuous, time-course, and censored time-to-event targets}.

Conditional independence of $X$ with $T$ given $\mathbf{S}$ implies that $P(T | \mathbf{S}, X) = P(T | \mathbf{S})$, whenever $P(\mathbf{S}) > 0$ ($\mathbf{S}$ is allowed to be the empty set).
Thus, when conditional independence holds, $X$ is not predictive of $T$ when $\mathbf{S}$ (and only $\mathbf{S}$) is known.
A conditional independence test assumes the null hypothesis that feature $X$ is probabilistically independent of $T$ (i.e., redundant) given a set of variables $\mathbf{S}$ and is denoted by $\mathit{Test}(X, T | \mathbf{S})$. 
The test returns a $p$-value, which corresponds to the probability that one obtains deviations from what is expected under the null hypothesis as extreme or more extreme than the deviation actually observed with the given data. 
When $p_k \equiv \mathit{Test}(X_k, T | \mathbf{S})$ is low, the null hypothesis can be safely rejected: the value of $X_k$ {\em does provide predictive information} for $T$ when the values of {\bf S} are known.
In practice, decisions are made using a threshold $\alpha$ (significance level) on the $p$-values; the null hypothesis is rejected if the $p$-value is below $\alpha$.

In the context of feature selection, the $p$-values $p_k$ returned by statistical hypotheses tests of conditional independence are employed not only to reject or accept hypotheses, but also {\em to rank the features according to the predictive information they provide for $T$} given $\mathbf{S}$.
Intuitively, this can be justified by the fact that everything else being equal (i.e., sample size, type of test) the $p$-values of such tests in case of dependence have (on average) the reverse ordering with the conditional association of the variables with $T$ given $\mathbf{S}$. So, the basic variant of the algorithm selects to add (remove) the feature with the lower (higher) $p$-value in each Forward (Backward) Iteration.
The Forward-Backward Selection algorithm using conditional independence tests is summarized in Algorithm~\ref{alg:fbs}.
We use $V_\mathcal{D}$ to denote the set of variables contained in dataset $\mathcal{D}$ (excluding $T$).
The $\textproc{Pvalue}(T,X|\mathbf{S})$ function performs a conditional independence test of $T$ and $X$ given $\mathbf{S}$ and returns a $p$-value.

\subsection{Implementing Independence Tests using the Likelihood Ratio Technique}
There are several methods for assessing conditional independence, such as likelihood-ratio based tests (or asymptotically equivalent approximations thereof like score tests and Wald tests \citep{Engle1984}) or kernel-based tests \cite{Zhang2011}.
We focus on likelihood-ratio based tests hereafter, mostly because they are general and can be applied for different data types (e.g. continuous, ordinal, nominal, time-to-event, to name a few), although the main algorithm is not limited to such tests but can be applied with any type of test.

To construct a likelihood-ratio test for conditional independence of $T$ with $X$ given $\mathbf{S}$ one needs a statistical model that maximizes the log-likelihood of the data $LL(D; \theta) \equiv \log P(D | \theta)$ over a set of parameters $\theta$. Without loss of generality, we assume hereafter $T$ is binary and consider the binary logistic regression model. For the logistic regression, the parameters $\theta$ are weight coefficients for each feature in the model and an intercept term. Subsequently, two statistical models have to be created for $T$: (i) model $M_0$ using only variables $\mathbf{S}$, and (ii) model $M_1$ using $\mathbf{S}$ and $X$ resulting in corresponding log-likelihoods $LL_0$ and $LL_1$. The null hypothesis of independence now becomes equivalent to the hypothesis that both log-likelihoods are equal asymptotically. The test statistic function of the test is called the {\em deviance} and is defined as $$D \equiv 2\times (LL_0 - LL_1)$$ 

Notice that, the difference in the logs of the likelihoods corresponds to the ratio of the likelihoods, hence the name likelihood-ratio test. The test statistic is known to follow asymptotically a $\chi^2$ distribution with $P_1 - P_0$ degrees of freedom \citep{Wilks1938}, where $P_1$ and $P_0$ are degrees of freedom of models $M_1$ and $M_0$ respectively\footnote{An implicit assumption made here is that the models are correctly specified. If this does not hold, the statistic follows a different distribution \citep{Foutz1977}. There exist methods that handle the more general case \cite{White1982, Vuong1989}, but this is clearly out of this paper's scope.}. When $X$ is a continuous feature, only one more parameter is added to $\theta$ so the difference in degrees of freedom is 1 for this case. Categorical predictors can be used by simply encoding them as $K-1$ dummy binary features, where $K$ is the number of possible values of the original feature. In this case, the difference in degrees of freedom is $K-1$. Knowing the theoretical distribution of the statistic allows one to compute the $p$-value of the test: $p = 1 - \mathit{cdf}(D, df)$, where $\mathit{cdf}$ is the cumulative probability distribution function of the $\chi^2$ distribution with degrees of freedom $df$ and $D$ the observed deviance. Likelihood-ratio tests can be constructed for any type of data for which an algorithm for maximizing the data likelihood exists, such as binary, multinomial or ordinal logistic regression, linear regression and Cox regression to name a few. 

Likelihood-ratio tests are {\em approximate} in the sense that the test statistic has a $\chi^2$ distribution only {\em asymptotically}. When sample size is low, the asymptotic approximation may return inaccurate $p$-values. Thus, {\em to apply approximate tests it is important to ensure a sufficient number of samples is available}. This issue is treated in detail in the context of PBFP and the logistic test in Section \ref{app:minsample}. Note that, the aforementioned models and the corresponding independence tests are only suited for identifying linear dependencies; certain types of non-linear dependencies may also be identifiable if one also includes interaction terms and feature transformations in the models. 

\subsection{Combining p-values Using Meta-Analysis Techniques}
\label{sec:combine}
A set of $p$-values stemming from testing the {\em same} null hypothesis (e.g. testing the conditional independence of $X$ and $Y$ given $\mathbf{Z}$) can be combined using statistical meta-analysis techniques into a single $p$-value. Multiple such methods exist in the literature \citep{Loughin2004}. Fisher's combined probability test \cite{Fisher1932} is one such method that has been shown to work well across many cases \citep{Loughin2004}. It assumes that the $p$-values are independent and combines them into a single statistic using the formula
$$ \text{Statistic} \equiv -2\sum_{i = 1}^{K} \log(p_i) $$
where $K$ is the number of $p$-values, $p_i$ is the $i$-th $p$-value, and $\log$ is the natural logarithm. The statistic is then distributed as a $\chi^2$ random variable with $2\cdot K$ degrees of freedom, from which a combined $p$-value is computed. 

\subsection{Bootstrap-based Hypothesis Testing}
The bootstrap procedure \cite{Efron1994} can be used to compute the distribution of a statistic of interest. Bootstrapping is employed in the PFBP algorithm for making early, probabilistic decisions. Bootstrapping is a general-purpose non-parametric resampling-based procedure which works as follows: (a) resample with replacement from the input values a sample of equal size, (b) compute the statistic of interest on the bootstrap sample, (c) repeat steps (a) and (b) many times to get an estimate of the bootstrap distribution of the statistic.
The bootstrap distribution can then be used to compute properties of the distribution such as confidence intervals, or to compute some condition of the statistic; a simple example application on the latter follows.

Let $\mu_X$ denote the mean of random variable $X$ and let $\hat{\mu}_X$ denote the estimate of the mean of $X$ given a sample of $X$.
Assume we are given a sample of size $n$ of random variable $X$ and we want to compute the probability that the mean of $X$ is larger than $10$, $P(\mu_X > 10)$.
That probability is a Bernoulli random variable, and the statistic in this case is a binary valued variable (i.e., taking a value of 0 or 1 with probability $P(\mu_X > 10)$).
Using bootstrapping, $P(\mu_X > 10)$ can be estimated as follows: (a) sample with replacement $n$ values of $X$ and create the $b$-th bootstrap sample $X^b$, (b) estimate the mean of $X^b$, denoted as $\hat{\mu}^b_X$, and compute $I(\hat{\mu}^b_X > 10)$, where $I$ is the indicator function returning 1 if the inequality holds and 0 otherwise, and (c) repeat (a) and (b) $B$ times (e.g. $B = 1000$).
$P(\mu_X > 10)$ is then computed as
$$
P(\mu_X > 10) = \frac{I(\hat{\mu}_X > 10) + \sum_{i = 1}^{B} I(\hat{\mu}_X^b > 10)}{B+1}
$$
Note that, we also compute the statistic on the original sample, and thus divide by $B+1$.

\subsection{Probabilistic Graphical Models and Markov Blankets}
In this section, we give a brief overview of Bayesian networks and maximal ancestral graphs, which will be used later on to present the theoretical properties of the proposed algorithm.
A more extensive exposition and rigorous treatment can be found in \cite{Spirtes2000,SpirtesRichardson2002,Aliferis2010JMLR}.

\subsubsection{Bayesian Networks}
A Bayesian network $B = \langle G, P \rangle$ consists of a directed acyclic graph $G$ over a set of vertices $V$ and a joint distribution $P$, over random variables that correspond one-to-one to vertices in $V$ (thus, no distinction is made between variables and vertices). The \textbf{Markov condition} has to hold between $G$ and $P$: every variable $X$ is conditionally independent of its non-descendants in $G$, given its parents, denoted by $Pa(X)$. The Markov condition leads to a factorization of the joint probability $P(V) = \prod_i P(X_i | Pa(X_i))$. Thus, the graph $G$ determines a factorization of the probability distribution, directly implying that some independencies have to hold, and further entailing (along with the other probability axioms) some additional conditional independencies. A Bayesian network is called \textbf{faithful} if all and only the conditional independencies in $P$ are entailed by the Markov condition. Conceptually, this faithfulness condition means that all independencies in the distribution of the data are determined by the structure of the graph $G$ and not the actual parameterization of the distribution. A distribution $P$ is called \textbf{faithful} (to a Bayesian Network) if there is a graph $G$ such that $B = \langle G, P \rangle$ is faithful. Under the Markov and faithfulness assumptions, a graphical criterion called \textbf{d-separation} \cite{Verma1988,Pearl1988} can be used to read off dependencies and independencies encoded in a Bayesian network. To define $d$-separation the notion of {\em colliders} is used, which are triplets of variables $\langle X, Y, Z\rangle$ with $X$ and $Z$ having directed edges into $Y$. Two variables $X$ and $Y$ are $d$-connected by a set of variables $\mathbf{Z}$ if and only if there exists a (not necessarily directed) path $p$ between $X$ and $Y$ such that (i) for each collider $V$ on $p$, $V$ is either in $\mathbf{Z}$ or some descendant of $V$ is in $\mathbf{Z}$, and (ii) no non-collider on $p$ is in $\mathbf{Z}$. In case no such path exists, $X$ and $Y$ are $d$-separated given $\mathbf{Z}$. Thus, the Markov and faithfulness conditions imply that if two variables $X$ and $Y$ are $d$-separated ($d$-connected) given $\mathbf{Z}$, then they are conditional independent (dependent) given $\mathbf{Z}$.

\subsubsection{Maximal Ancestral Graphs}
A distribution class strictly larger than the set of faithful distributions to Bayesian Networks, is the set of distributions that are marginals of faithful distributi0ons. Unfortunately, marginals of faithful distributions are not always faithful to some Bayesian network! 
Thus, marginalization over some variables loses the faithfulness property: the marginal distribution cannot always be faithfully represented by a Bayesian network. However, faithful marginal distributions can be represented by another type of graph called \textbf{directed maximal ancestral graph} \cite{SpirtesRichardson2002} or DMAG. DMAGs include not only directional edges, but also bi-directional edges. DMAGs are extensions of Bayesian networks for marginal distributions and are closed under marginalization. The representation of a marginal of a faithful (to a Bayesian network) distribution by a DMAG is again faithful, in the sense that all and only the conditional independencies in the distribution are implied by the Markov condition. The set of conditional independencies entailed by a DMAG is provided by a criterion similar to $d$-separation, now called $m$-separation. 

\subsubsection{Markov Blankets in Probabilistic Graphical Models} A \textbf{Markov blanket} of $T$ with respect to a set of variables $\mathbf{V}$ is defined as a minimal set $\mathbf{S}$ such that $\condind{\mathbf{V} \setminus \mathbf{S}}{T}{\mathbf{S}}$, where $\condind{\mathbf{X}}{T}{\mathbf{S}}$ denotes the conditional independence of $\mathbf{X}$ with $T$ given $\mathbf{S}$. Thus, a Markov blanket of $T$ is any minimal set that renders all other variables conditionally independent. An important theorem connects the Markov blanket of $T$ with the feature selection problem for $T$: {\em under broad conditions \cite{Margaritis2000,Tsamardinos2003} a Markov blanket of $T$ is a solution to the feature selection problem for $T$}. When the distribution is faithful to a Bayesian network or DMAG, the Markov blanket of $T$ is unique \footnote{Some recent algorithms \cite{Statnikov2013,MXM16} deal with the problem of solution multiplicity in feature selection.}. In other words, for faithful distributions, the Markov Blanket of $T$ has a direct graphical interpretation. The Markov blanket consists of all vertices adjacent to $T$, and all vertices that are reachable from $T$ through a collider path, which is a path where all vertices except the start and end vertices are colliders \cite{Borboudakis2017}. For Bayesian networks, this corresponds to the set of parents (vertices with an edge to $T$), children (vertices with an edge from $T$), and spouses (parents of children) of $T$ in $G$.

\section{Massively Parallel Forward-Backward Algorithm}
\label{sec:alg}
We provide an overview of our algorithm, called Parallel, Forward-Backward with Pruning ({\bf PFBP}), an extension of the basic Forward-Backward Selection (FBS) algorithm (see Section~\ref{sec:back:fbs} for a description).
We will use the terminology introduced for FBS: a forward (backward) \textbf{Phase} refers to the forward (backward) loops of FBS, and an \textbf{Iteration} refers to each loop iteration that decides which variable to select (remove) next.
PFBP is presented in ``evolutionary'' steps where successive enhancements are introduced in order to make computations local or reduce computations and communication costs; the complete algorithm is presented in Section~\ref{sec:pfbp}. 
To evaluate predictive performance of candidate features we use $p$-values of conditional independence tests, as described in Section~\ref{sec:back:fbs}.
We assume the data are provided in the form of a 2-dimensional matrix $\mathcal{D}$ where rows correspond to training instances (samples) and columns to features (variables), and one of the variables is the target variable $T$.
Physically, the data matrix is partitioned in sub-matrices $D_{i,j}$ and stored in a distributed fashion in {\em workers} in a cluster running Spark \cite{Zaharia2010} or similar platform.
Workers perform in parallel local computations on each $D_{i,j}$ and a {\em master} node performs the centralized, global computations.

\subsection{Data Partitions in Blocks and Groups and Parallelization Strategy}
\begin{figure*}[!ht]
\centering
\includegraphics[width=1\textwidth]{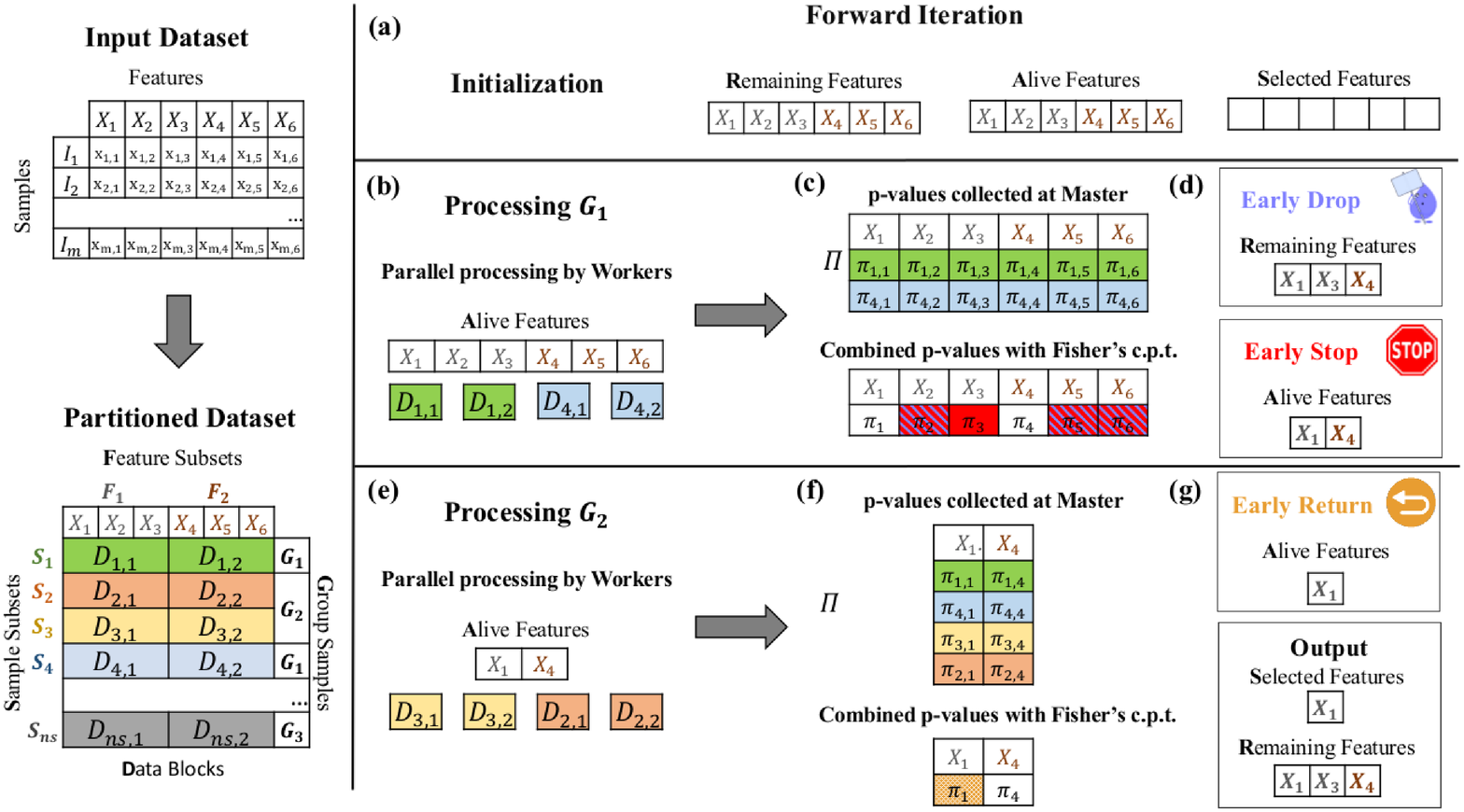}
\caption{ {\bf Left}: Data partitioning of the algorithm. In the top the initial data matrix $\mathcal{D}$ is shown with 6 features and instances $I_1, \ldots, I_m$. In the bottom, the 6 features are partitioned to Feature Subsets $F_1 = \{1, 2, 3\}$ and $F_2 = \{4, 5, 6\}$. The rows are randomly partitioned to Sample Subsets $S_1, \ldots, S_{\mathit{ns}}$, and the Sample Subsets are assigned to Group Samples. Each Block $D_{i,j}$ is physically stored as a unit. {\bf Right}: Example of trace of a Forward Iteration of PFBP. (a) The $\mathbf{R}$emaining features, $\mathbf{A}$live features, and $\mathbf{S}$elected features are initialized.
(b) All Data Blocks $D_{1,1}, D_{1,2}, D_{4,1}, D_{4,2}$ in the first Group are processed in parallel (by workers). 
(c) The resulting local $p$-values are collected (reduced) in a master node for each Alive feature and Sample Set in the first Group (as well as the likelihoods, not shown in the Figure). (d) Bootstrap-based tests determine which features to Early Drop or Stop based on $\Pi$, or whether to Early Return (based on $\Lambda$, not shown in the Figure). The sets $\mathbf{R}$ and $\mathbf{A}$ are updated accordingly. In this example, $X_2$, $X_5$ and $X_6$ are Dropped, and only $X_1$ and $X_4$ remain Alive. Notice that always $\mathbf{A} \subseteq \mathbf{R}$. 
(e) The second Group is processed in parallel (by workers) containing Blocks $D_{3,1}, D_{3,2}, D_{2,1}, D_{2,2}$. 
(f) New local $p$-values for all features still Alive are appended to $\Pi$. If $G_2$ was the last Group, global $p$-values for the Alive features would be computed and the one with the minimum value (in this example $X_1$) would be selected for inclusion in {\bf S}. (g) In case, $X_1$ and $X_4$ are deemed almost equally predictive (based on their log-likelihoods) the current best is Early Returned.}
\label{fig:trace}
\end{figure*}

We now describe the way $\mathcal{D}$ is partitioned in sub-matrices to enable parallel computations. 
First, the set of available features (columns) $\mathbf{F}$ is partitioned to about equal-sized \textbf{Feature Subsets} $\{F_1, \ldots, F_\mathit{nf}\}$. 
Similarly, the samples (rows) are randomly partitioned to about equal-sized \textbf{Sample Subsets} $\{S_1, \ldots, S_\mathit{ns}\}$. 
The row and column partitioning defines sub-matrices called \textbf{Data Blocks} $D_{i,j}$ with rows $S_i$ and features $F_j$. 
Sample Subsets are assigned to $Q$ \textbf{Group Samples} $\{G_q\}_1^C$ of size $C$ each, where each group sample $G_q$ is a set $\{S_{q_1}, \ldots, S_{q_n}\}$ (i.e., the set of Sample Subsets is partitioned). 
The Data Blocks $D_{i,j}$ with samples within a group sample $S_i \in G_q$ belong in the same {\bf Group}. 
This second, higher level of grouping is required by the bootstrap tests explained in Section~\ref{sec:pruning}.
Data Blocks in the same Group are processed in parallel in different workers (provided enough are available). 
However, Groups are processed sequentially, i.e., computation in all Blocks within a Group has to complete to begin computations in the Blocks of the next Group. 
Obviously, if workers are more than the Data Blocks, there is no need for defining Groups. 
The data partitioning scheme is shown in Figure~\ref{fig:trace}:Left.
Details of how the number of Sample Sets $\mathit{ns}$, the number of Feature Subsets $\mathit{nf}$, and the number $C$ of Group Samples are determined are provided in Section~\ref{sec:parameters}.

\subsection{Approximating Global $p$-values by Combining Local $p$-values Using Meta-Analysis}
Recall that Forward-Backward Selection uses $p$-values stemming from conditional independence tests to rank the variables and to select the best on for inclusion (forward Phase) or exclusion (backward Phase).
Extending the conditional independence tests to be computed over multiple Data Blocks is not straightforward, and may be computationally inefficient.
For conditional independence tests based on regression models (e.g. logistic or Cox regression), a maximum-likelihood estimation over all samples has to be performed, which typically does not have a closed-form solution and thus requires the use of an iterative procedure (e.g. Newton descent).
Due to its iterative nature, it results in a high communication cost rendering it computationally inefficient, especially for feature selection purposes on Big Data, as many models have to be fit at each Iteration.

Instead of fitting full (global) regression models, we propose to perform the conditional independence tests locally on each data block, and to combine the resulting $p$-values using statistical meta-analysis techniques.
Specifically, the algorithm computes \textit{local $p$-values denoted by $\pi_{i,k}$} for candidate feature $X_k$ from only the rows in $S_i$ of a data block $D_{i,j}$, where $F_j$ contains the feature $X_k$. 
This enables massive parallelization of the algorithm, as each data block can be processed independently and in parallel by a different worker.
The local $p$-values $\pi_{i,k}$ are then \textit{communicated} to the master node of the cluster, and are stored in a matrix $\Pi$; we will use $\pi_{i,k}$ to refer to the elements of matrix $\Pi$, corresponding to the local $p$-value of $X_k$ computed on a data block containing samples in sample set $S_i$.
Using the $p$-values in matrix $\Pi$, the master node combines the $p$-values to \textit{global $p$-values} for each feature $X_k$ using Fisher's combined probability test \cite{Fisher1932} (Fig.~\ref{fig:trace}:Right(b)) \footnote{Naturally, any method for combining $p$-values can be used instead of Fisher's method, but we did not further investigate this in this work.}.
Finally, we note that this approach is not limited to regression-based tests, but can be used with any type of conditional independence test, and is most appropriate for tests which are hard to parallelize, or computationally expensive (e.g. kernel-based tests \citep{Zhang2011}).

Using Fisher's combined probability test to combine local $p$-values does not necessarily lead to the same $p$-value as the one computed over all samples.
There are no guarantees how close those $p$-values will be in case the null hypothesis of conditional independence holds, except that they are uniformly distributed between 0 and 1.
In case the null hypothesis does not hold however, one expects to reject the null hypothesis using either method in the sample limit.
What matters in practice is the fact that PFBP makes often the same decision at each Iteration, that is, the top ranked variable is often the same.
Note that, even if the top ranked variable is not the same one, PFBP may still perform well, as long as some other informative variable is ranked first.
The accuracy of Fisher's combined probability test is further investigated in experiments on synthetic data, presented in Section~\ref{app:minsample}, where we show that, if the sample size per data block is sufficiently high, choosing a value by combining $p$-values leads to the same decision.

For the computation of the local $p$-values on $D_{i,j}$, samples $S_i$ of the selected features $\mathbf{S}$ are required, and thus the data need to be broadcast to every worker processing $D_{i,j}$ whenever $\mathbf{S}$ is augmented, i.e., in the end of each Forward Iteration. 
In total, the communication cost of the algorithm is due to the assembly of all local $p$-values $\pi_{i,k}$ to determine the next feature to include (exclude), as well as the broadcast of the data for the newly added feature in $\mathbf{S}$ at the end of each forward Iteration. 
\textit{We would like to emphasize that the bulk of computation of the algorithm is the calculation of local $p$-values that require expensive statistical tests and it takes place in the workers in parallel. The central computations in the master are minimal}.

\subsection{Speeding-up PFBP using Pruning Heuristics}
In this section, we present 3 pruning heuristics used by PFBP to speed-up computation.
Implementation details of the heuristics using locally computed $p$-values are presented in Section~\ref{sec:pruning}.

\subsubsection{Early Dropping of Features from Subsequent Iterations}
The first addition to PFBP is the {\bf Early Dropping} (ED) heuristic, first introduced in \citep{Borboudakis2017} for a non-parallel version of Forward-Backward Selection.
Let $\mathbf{R}$ denote the set of remaining features, that is, the set of features still under consideration for selection.
Initially, $\mathbf{R} = \mathbf{F} \setminus \mathbf{S}$, where $\mathbf{F}$ is the set of all available features and $\mathbf{S}$ is the set of selected features, which is initially empty.
At each forward Iteration, ED removes from $\mathbf{R}$ all features that are conditionally independent of the target $T$ given the set of currently selected features $\mathbf{S}$.
Typically, just after the first few Iterations of PFBP, only a very small proportion of the features will still remain in $\mathbf{R}$, leading to orders of magnitude of efficiency improvements even in the non-parallel version of the algorithm \citep{Borboudakis2017}.
When the set of variables $\mathbf{R}$ becomes empty, we say that PFBP finished one $\textbf{Run}$.
Unfortunately, the Early Dropping heuristic without further adjustments may miss important features which seem uninformative at first, but provide information for $T$ when considered with features selected in subsequent Iterations.
Variables should be given additional opportunities to be selected by performing more Runs. Each additional Run calls the forward phase again but starts with the previously selected variables $\mathbf{S}$ and re-initializes the remaining variables to $\mathbf{R} = \mathbf{F} \setminus \mathbf{S}$.
By default, PFBP uses 2 Runs, although a different number of Runs may be used.
Typically a value of 1 or 2 is sufficient in practice, with larger values requiring more computational time while also giving stronger theoretical guarantees; the theoretical properties of PFBP with ED are described in Section~\ref{sec:theory} in detail; in short, PFBP with 2 Runs and assume no statistical errors in the conditional independence tests returns the Markov Blanket of $T$ is distributions faithful to a Bayesian Network. \textit{Overall, by discarding variables at each Iteration, the Early Dropping heuristic allows the algorithm to scale with feature size.}

\subsubsection{Early Stopping of Features within the Same Iteration}
The next addition to the algorithm regards \textbf{Early Stopping} (ES) of consideration of features \textit{within the same} Iteration, i.e., in order to select the next best feature to select in a forward Iteration or to remove in a backward Iteration.
To implement ES we introduce the set $\mathbf{A}$ of features still \emph{Alive} (i.e., under consideration) in the current Iteration, initialized to $\mathbf{A} = \mathbf{R}$ at the beginning of each Iteration (see Figure~\ref{fig:trace}:Right(b)). 
As the master node gathers local $p$-values for a feature $X_k$ from several Data Blocks, it may be able to determine that no more local $p$-values need to be computed for $X_k$. 
This is the case if these $p$-values are enough to safely decide that with high probability $X_k$ is not going to be selected for inclusion (Forward Phase) or exclusion (Backward Phase) in this Iteration. 
In this case, $X_k$ is removed from the set of alive features $\mathbf{A}$, and is not further considered in the current Iteration.
\textit{This allows PFBP to quickly filter out variables which will not be selected at the current Iteration.
Thus, ES leads to a super-linear speed-up of the feature selection algorithm with respect to the sample size: even if samples are doubled, the same features will be Early Stopped; $p$-values will not be computed for these features on the extra samples}. 

\subsubsection{Early Return of the Winning Feature}
The final heuristic of the algorithm is called \textbf{Early Return} (ER).
Recall that Early Dropping will remove features conditionally independent of $T$ given $\mathbf{S}$ from this and subsequent Iterations while Early Stopping will remove non-winners from the current Iteration. 
However, even using both heuristics, the algorithm will keep producing local $p$-values for features $X_j$ and $X_k$ that are candidates for selection and at the same time are informationally indistinguishable (equally predictive given $\mathbf{S}$) with regards to $T$ (this is the case when the residuals of $X_j$ and $X_k$ given  $\mathbf{S}$ are almost collinear). 
When two or more features are both candidates for selection and almost indistinguishable, it does not make sense to go through the remaining data: all choices are almost equally good. 
Hence, Early Return terminates the \textit{computation} in the current Iteration and returns the current best feature $X_j$, if with high probability it is not going to be much worse than the best feature in the end of the Iteration (see Fig.~\ref{fig:trace}: Right(c)). 
Again, the result is that computation in the current Iteration may not process all Groups. 
The motivation behind Early Return is similar to Early Stopping, in that it tries to quickly determine the next variable to select.
The difference is that, Early Return tries to quickly determine whether a variable is ``good enough'' to be selected, in contrast to Early Stopping which discards unpromising variables.

A technical detail is that judging whether two features $X_i$ and $X_j$ are ``equally predictive" is implemented using the log-likelihoods $\lambda_i$ and $\lambda_j$ of the models with predictors $\mathbf{S} \cup \{X_i\}$ and $\mathbf{S} \cup \{X_j\}$ instead of the corresponding $p$-values. 
The likelihoods are part of the computation of the $p$-values, thus incur no additional computational overhead. 

\subsection{The Parallel Forward-Backward with Pruning Algorithm}
\label{sec:pfbp}

\begin{algorithm}[t!]
\caption{\textbf{Parallel Forward-Backward With Pruning} (\textproc{PFBP})}
\label{alg:fbs:general}
\begin{algorithmic}[1]
	\Require Dataset $\mathcal{D}$, Target $T$, Maximum Number of Runs maxRuns
    \Ensure Selected Variables $\mathbf{S}$
    \State \Comment{\textit{Data Partitioning}}
	\State \textit{Randomly assign samples to sample sets $S_1, \dots, S_{\mathit{ns}}$}
    \State \textit{Assign sample sets $S_1, \dots, S_{\mathit{ns}}$ to equally-sized Groups $G_1, \dots, G_K$}
	\State \textit{Assign features to feature sets $F_1, \dots, F_{nf}$}
	\State \textit{Partition $\mathcal{D}$ to data blocks $D_{i,j}$ containing samples from $S_i$ and $F_j$, $\forall i,j$}
    \State 
	\State $\mathbf{S} \gets \emptyset$ \Comment{\textit{No selected variables}}
    \State $run \gets 1$ \Comment{\textit{First run}}
    \State
    \State \Comment{\textit{Iterate until (a) maximum number of runs reached, or (b) selected features $\mathbf{S}$ did not change}}
	\While{$run \leq maxRuns$ $\wedge$ $\mathbf{S}$ changes}
    \State $\mathbf{S} \gets$ \textproc{OneRun}($D$, $T$, $\mathbf{S}$)
    \State $run \gets run + 1$
   	\EndWhile
	\State {\bfseries return} $\mathbf{S}$
    \algrule
    \State \textbf{function} \textproc{OneRun}(Data Blocks $D$, Target $T$, Selected Variables $\mathbf{S}$, Maximum Number of Variables To Select maxVars)
	\State $\mathbf{R} \gets \mathbf{F} \setminus \mathbf{S}$ \Comment{\textit{All variables remaining}}
    \State \Comment{\textit{Forward phase: iterate until (a) maximum number of variables selected or (b) no new variable has been selected}}
	\While{$|\mathbf{S}| < \mathit{maxVars}$ $\wedge$ $\mathbf{S}$ changes}
    \State $\langle \mathbf{S}, \mathbf{R}\rangle \gets$ \textproc{ForwardIteration}($D$, $T$, $\mathbf{S}$, $\mathbf{R}$)
   	\EndWhile
    \State
	\State \Comment{\textit{Backward phase: iterate until no variable can be removed}}
	\While{$\mathbf{S}$ changes}
    \State $\mathbf{S} \gets$ \textproc{BackwardIteration}($D$, $T$, $\mathbf{S}$)
   	\EndWhile
	\State {\bfseries return} $\mathbf{S}$
    \end{algorithmic}
\end{algorithm}

We present the proposed Parallel Forward-Backward with Pruning (PFBP) algorithm, shown in Algorithm~\ref{alg:fbs:general}. 
To improve readability, several arguments are omitted from function calls.
PFBP takes as input a dataset $\mathcal{D}$ and the target variable of interest $T$.
Initially the number of Sample Sets $\mathit{ns}$ and number of Feature Sets $\mathit{nf}$ are determined as described in Section~\ref{sec:parameters}.
Then, (a) the samples are randomly assigned to Sample Sets $S_1, \dots, S_{\mathit{ns}}$, to avoid any systematic biases (see also Section~\ref{sec:pfbp:practical:fisher}), (b) the Sample Sets $S_1, \dots, S_{\mathit{ns}}$ are assigned to $Q$ approximately equal-sized Groups, $G_1, \dots, G_Q$, (c) the features are assigned to feature sets $F_1, \dots, F_{\mathit{nf}}$, in order of occurrence in the dataset, and (d) the dataset $\mathcal{D}$ is partitioned into data blocks $D_{i,j}$, with each such block containing samples and features corresponding to sample set $S_i$ and feature set $F_j$ respectively.
The selected variables $\mathbf{S}$ are initialized to the empty set.
The main loop of the algorithm performs up to $\mathit{maxRuns}$ Runs, as long as the selected variables $\mathbf{S}$ change.
Each such Run executes a forward and a backward Phase.

The $\textproc{OneRun}$ function takes as input a set of data blocks $D$, the target variable $T$, a set of selected variables $\mathbf{S}$, and a limit on the number of variables to select $\mathit{maxVars}$.
It initializes the set of remaining variables $\mathbf{R}$ to all non-selected variables $\mathbf{F} \setminus \mathbf{S}$.
Then, it executes the forward and backward Phases.
The forward (backward) Phase executes forward (backward) Iterations until some stopping criteria are met.
Specifically, the forward Phase terminates if the maximum number of variables $\mathit{maxVars}$ has been selected, or until no more variable can be selected, while the backward Phase terminates if no more variables can be removed from $\mathbf{S}$.

\begin{algorithm}[t!]
\caption{\textproc{ForwardIteration}}
\label{alg:fbs:forwarditeration}
\begin{algorithmic}[1]
	\Require Data Blocks $D$, Target $T$, Selected Variables $\mathbf{S}$, Remaining Variables $\mathbf{R}$
    \Ensure Selected Variables $\mathbf{S}$, Remaining Variables $\mathbf{R}$
	\State $\mathbf{A}$ $\gets$ $\mathbf{R}$ \Comment{\textit{Initialize Alive Variables}}
    \State $\Pi$ \Comment{\textit{Array of log-pvalues, initially empty}}
    \State $\Lambda$ \Comment{\textit{Array of log-likelihoods, initially empty}}
    \State $q \gets 1$ \Comment{\textit{Initialize current Group counter}}
    \State $Q \gets \#Groups$ \Comment{\textit{Set Q to the total number of Groups}}
    \State
    \While{$q \leq Q$}
    \State \Comment{\textit{Process the alive features $\mathbf{A}$ for all data blocks containing sample sets in $G_q$ (denoted as $D_q$) in parallel in workers for the given $T$, $\mathbf{S}$ and $\mathbf{A}$, compute sub-matrices $\Pi_{q}$ and $\Lambda_{q}$ from each block, and append results to $\Pi$ and $\Lambda$}}
    \State $\langle\Pi_{q}, \Lambda_{q}\rangle \gets$ \textproc{TestParallel}$(D_q, T, \mathbf{S}, \mathbf{A})$
    \State $\langle \mathbf{R}, \mathbf{A} \rangle$ $\gets$ \textproc{EarlyDropping}$(\Pi, \mathbf{R}, \mathbf{A})$
    \State $\mathbf{A}$ $\gets$ \textproc{EarlyStopping}$(\Pi, \mathbf{A})$
    \State $\mathbf{A}$ $\gets$ \textproc{EarlyReturn}$(\Lambda, \mathbf{A})$
    \State Update $\Pi$ and $\Lambda$ (Retain only columns of alive variables)
    \State \Comment{\textit{Stop if single variable alive}}
    \If{$|\mathbf{A}| \leq 1$}
    \State {\bfseries break}
    \EndIf
    \State $q \gets q + 1$
    \EndWhile
    \State 
    \If{$|\mathbf{A}| > 0$}
      \State $\pi \gets$ \textproc{Combine}$(\Pi)$ \Comment{\textit{Compute final combined $p$-value for all alive variables}}
      \State $X_{best}$ $\gets$ $\argmin\limits_{X_i \in \mathbf{A}}$ $\pi_i$ \Comment{\textit{Identify the best variable $X_{best}$}}
      \State \Comment{\textit{Select $X_{best}$ if dependent with $T$ given $\mathbf{S}$}}
      \If{$\pi_{best}$ $\leq$ $\alpha$}
      \State $\mathbf{S}$ $\gets$ $\mathbf{S} \cup \{X_{best}\}$
      \State $\mathbf{R}$ $\gets$ $\mathbf{R} \setminus \{X_{best}\}$
      \EndIf
    \EndIf
    \State {\bfseries return} $\langle \mathbf{S}, \mathbf{R} \rangle$
\end{algorithmic}
\end{algorithm}

\begin{algorithm}[t!]
\caption{\textproc{BackwardIteration}}
\label{alg:fbs:backwarditeration}
\begin{algorithmic}[1]
	\Require Data Blocks $D$, Target $T$, Selected Variables $\mathbf{S}$
    \Ensure Selected Variables $\mathbf{S}$, Remaining Variables $\mathbf{R}$
	\State $\mathbf{A}$ $\gets$ $\mathbf{S}$ \Comment{\textit{Initialize Alive Variables}}
    \State $\Pi$ \Comment{\textit{Array of log-pvalues, initially empty}}
    \State $q \gets 1$ \Comment{\textit{Initialize current Group counter}}
    \State $Q \gets \#Groups$ \Comment{\textit{Set Q to the total number of Groups}}
    \State
    \While{$q \leq Q$}
    \State \Comment{\textit{Process the alive features $\mathbf{A}$ for all data blocks containing sample sets in $G_q$ (denoted as $D_q$) in parallel in workers for the given $T$, $\mathbf{S}$ and $\mathbf{A}$, compute sub-matrix $\Pi_q$ from each block, and append it to $\Pi$}}
    \State $\Pi_{q} \gets$ \textproc{TestParallel}$(D_q, T, \mathbf{S}, \mathbf{A})$
    \State $\mathbf{A}$ $\gets$ \textproc{EarlyStoppingBackward}$(\Pi, \mathbf{A})$
    \State Update $\Pi$ (Retain only columns of alive variables)
    \State \Comment{\textit{Stop if single variable alive}}
    \If{$|\mathbf{A}| \leq 1$}
    \State {\bfseries break}
    \EndIf
    \State $q \gets q + 1$
    \EndWhile
    \State 
    \If{$|\mathbf{A}| > 0$}
      \State $\pi \gets$ \textproc{Combine}$(\Pi_q)$ \Comment{\textit{Compute final combined $p$-value for all alive variables}}
      \State $X_{worst}$ $\gets$ $\argmax\limits_{X_i \in \mathbf{A}}$ $\pi_i$ \Comment{\textit{Identify the worst variable $X_{worst}$}}
      \State \Comment{\textit{Remove $X_{worst}$ if independent with $T$ given $\mathbf{S} \setminus X_{worst}$}}
      \If{$\pi_{worst}$ $>$ $\alpha$}
      \State $\mathbf{S}$ $\gets$ $\mathbf{S} \setminus \{X_{worst}\}$
      \EndIf
    \EndIf
    \State {\bfseries return} $\mathbf{S}$
\end{algorithmic}
\end{algorithm}

The forward and backward Iteration procedures are shown in Algorithms~\ref{alg:fbs:forwarditeration} and \ref{alg:fbs:backwarditeration}.
$\textproc{ForwardIteration}$ takes as input the data blocks $D$, the target variable $T$ as well as the current sets of remaining and selected variables, performs a forward Iteration and outputs the updated sets of selected and remaining variables.
It uses the variable set $\mathbf{A}$ to keep track of all alive variables, i.e. variables that are candidates for selection.
The arrays $\Pi$ and $\Lambda$ contain the local log $p$-values and log-likelihoods, containing $\mathit{ns}$ rows (one for each sample set) and $|\mathbf{A}|$ columns (one for each alive variable).
The values of $\Pi$ and $\Lambda$ are initially empty, and are filled gradually after processing Groups.
We use $D_q$ to denote all data blocks which corresponds to Sample Sets contained in Group $G_q$.
Similarly, accessing the values of $\Pi$ and $\Lambda$ corresponding to Group $q$ and variables $\mathbf{X}$ is denoted as $\Pi_q$ and $\Lambda_q$.

In the main loop, the algorithm iteratively processes Groups in a synchronous fashion, until all Groups have been processed or no alive variable remains.
The $\textproc{TestParallel}$ function takes as input the data blocks $D_q$ corresponding to the current Group $G_q$, and performs all necessary independence tests in parallel in workers.
The results, denoted as $\Pi_q$ and $\Lambda_q$ are then appended to the $\Pi$ and $\Lambda$ matrices respectively.
After processing a Group, the tests for Early Dropping, Early Stopping and Early Return are performed, using all local $p$-values computed up to Group $q$; details about the implementation of the \textproc{EarlyDropping}, \textproc{EarlyStopping} and \textproc{EarlyReturn} algorithms when data have only been partially processed are given in Section~\ref{sec:pruning}.
The values of non-alive features are then removed from $\Pi$ and $\Lambda$ (see also Figure~\ref{fig:trace}(f) for an example).
If only a single alive variable remains, processing stops.
Note that, this is not checked in the while loop condition, in order to ensure that at least one Group has been processed if the input set of remaining variables contains a single variable.
Finally, the best alive variable (if such a variable exists) is selected if it is conditionally dependent with $T$ given the selected variables $\mathbf{S}$.
Conditional dependence is determined by using the $p$-value resulting from combining all local $p$-values available in $\Pi$.
$\textproc{BackwardIteration}$ is similar to \textproc{ForwardIteration} with the exception that (a) the remaining variables are not needed, and thus no dropping is performed, (b) no early return is performed, and (c) the tests are reversed, i.e. the worst variable is removed.

\subsection{Massively-Parallel Predictive Modeling}
\label{sec:lrmodel}
The technique of combining locally computed $p$-values to global ones to massively parallelize computations, can be applied not only for feature selection, but also for predictive modeling. This way, at the end of the feature selection process one could obtain an approximate predictive model with no additional overhead! We exploit this opportunity in the context of independence tests implemented by logistic regression. During the computation of local $p$-values $\pi_{i,k}$ a (logistic) model for $T$ using all selected features $\mathbf{S}$ is produced from the samples in $S_i$. Such a model computes a set of coefficients $\beta_{i}$ that weighs each feature in the model to produce the probability that $T = 1$. 
Methods for combining multiple models, such as the ones considered here, are described in \cite{Becker2007}.
We used the weighted univariate least squares (WLS) approach \cite{Hedges1998}, with equal weights for each model; multivariate approaches may be more accurate and can also be applied in our case without any significant computational overhead, but were not further considered in this work. 
The WLS method with equal weights combines the local models to a global one $\hat{\beta}$ by just taking the average of the coefficient vectors of the model , i.e., $\hat{\beta} = \frac{1}{N}\sum_{i=1}^N \beta_i$. Thus, the only change to the algorithm is to cache each $\beta_{i}$ and average them in the master node. 
By default, PFBP uses the WLS method to construct a predictive model at each forward Iteration.

Using the previous technique, one could obtain a model at the end of each Iteration and assess its predictive performance (e.g., accuracy) on a hold-out validation set. Constructing for instance the graph of the number of selected features versus achieved predictive performance on the hold-out set could visually assist data analysts \cite{Konda2013} in determining how many features to include in the final selections; an example application on SNP data is given in the experimental section. An automated criterion for selecting the best trade-off between the number of selected features and the achieved predictive performance could also be devised, although this is out of the scope of this paper, as multiple testing has to be taken into consideration. 

\section{Implementation of the Early Dropping, Stopping and Return Heuristics using Bootstrap Tests on Local p-values}
\label{sec:pruning}
Recall that the algorithm processes Group Samples sequentially.
After processing each Group and collecting the results, PFBP applies the Early Dropping, Early Stopping and Early Return heuristics, computed on the master node, to filter out variables and reduce subsequent computation.
Thus, all three heuristics involve making early probabilistic decisions based on a subset of the samples examined so far.
Naturally, if all samples have been processed, Early Dropping can be applied on the combined $p$-values without making probabilistic decisions.

Before proceeding with the details, we provide the notation used hereafter.
Let $\Pi$ and $\Lambda$ be 2-dimensional arrays containing $K$ local log $p$-values and log-likelihoods for all alive variables in $\mathbf{A}$ and for all Groups already processed.
The matrices reside on the master node, and are updated each time a Group is processed.
Let $\pi_{i,j}$ and $\lambda_{i,j}$ denote the $i$-th value of the $j$-th alive variable, denoted as $X_j$.
Recall that those values have been computed locally on the Data Block containing samples from Sample Set $S_i$.
For the sake of simplicity, we will use $\pi_j$ and $\lambda_j$ ($l_j$) to denote the combined $p$-value and sum of log-likelihoods (likelihood) respectively of variable $X_j$.
The vectors $\pi$ and $\lambda$ will be used to refer to the combined $p$-values and sum of log-likelihoods for all alive variables respectively.
Also, let $X_{best}$ be the variable that would have been selected if no more data blocks were evaluated, that is, the one with the currently lowest combined $p$-value, denoted as $\pi_{best}$.

\subsection{Bootstrap Tests for Early Probabilistic Decisions}
In order to make early probabilistic decisions, we test: (a) $P(\pi_j \geq \alpha) > P_{drop}$ for Early Dropping of $X_j$ ($\alpha$ is the significance level), (b) $P(\pi_j > \pi_{best}) > P_{stop}$ for Early Stopping of $X_j$, and (c) $\forall X_j, (P(l_{best}/l_j \geq t)  > P_{return})$ for Early Return of $X_{best}$ (i.e., the probability is larger than the threshold for all variables), where $t$ is a tolerance parameter that determines how close the model with $X_{best}$ is to the rest in terms of how well it fits the data, and takes values between 0 and 1; the closer it is to 1, the closer it is in terms of performance to all other models.
By taking the logarithm, (c) can be rewritten as $\forall X_j, P(\lambda_{best} - \lambda_j \geq lt)$, where $lt = \log(t)$.

We employed bootstrapping to test the above.
A bootstrap-sample $b$ of $\Pi$ ($\Lambda$), denoted as $\Pi^b$ ($\Lambda^b$), is created by sampling with replacement $K$ \textbf{rows} from $\Pi$ ($\Lambda$).
Then, for each such sample, the Fisher's combined $p$-values (sum of log-likelihoods) are computed, by summing over all respective values for each alive variable; we refer to the vector of combined $p$-values (log-likelihoods) on bootstrap sample b as $\pi^b$ ($\lambda^b$), and the $i$-th element is referred to as $\pi_i^b$ ($\lambda_i^b$).
By performing the above $B$ times, probabilities (a), (b) and (c) can be estimated as:

\begin{equation}
  \tag{Early Dropping}
P(\pi_j \geq \alpha) = \frac{\textproc{I}(\pi_j \geq \alpha) + \sum_{b=1}^B \textproc{I}(\pi_j^b \geq \alpha)}{B + 1}
\end{equation}

\begin{equation}
  \tag{Early Stopping}
P(\pi_j > \pi_{best}) = \frac{\textproc{I}(\pi_j > \pi_{best}) + \sum_{b=1}^B \textproc{I}(\pi_j^b > \pi_{best}^b)}{B + 1}
\end{equation}

\begin{equation}
  \tag{Early Return}
P(\lambda_{best} - \lambda_j \geq lt) = \frac{\textproc{I}(\lambda_{best} - \lambda_j \geq lt) + \sum_{b=1}^B \textproc{I}(\lambda_{best}^b - \lambda^b_j \geq lt)}{B + 1}
\end{equation}

where $\textproc{I}$ is the indicator function, which evaluates to 1 if the inequality holds and to 0 otherwise.
For all of the above, the condition is also computed on the original sample, and the result is divided by the number of bootstrap iterations $B$ plus 1.
Note that, for Early Return the above value is computed for all variables $X_j$.

Algorithms~\ref{alg:drop},\ref{alg:stop} and \ref{alg:return} show the procedures in more detail.
For all heuristics, a vector named $cnts$ is used to keep track of how often the inequality is satisfied for each variable.
To avoid cluttering, the indicator function $\textproc{I}$ performs the check for multiple variables and returns a vector of values in each case, containing one value for each variable.
The function $\textproc{BootstrapSample}$ creates a bootstrap sample as described above, function $\textproc{Combine}$ uses Fisher's combined probability test to compute a combined $p$-value, and $\textproc{SumRows}$ sums over all rows of the log-likelihoods contained in $\Lambda$, returning a single value for each alive variable.

\begin{algorithm}[t!]
\caption{\textproc{EarlyDropping}}
\label{alg:drop}
\begin{algorithmic}[1]
    \Require Log $p$-values $\Pi$, Remaining Variables $\mathbf{R}$, Alive Variables $\mathbf{A}$, Number of Bootstrap Samples $B$, Significance Level Threshold $\alpha$, ED Threshold $P_{drop}$
    \Ensure Remaining variables $\mathbf{R}$, Alive Variables $\mathbf{A}$
    \State $\pi$ $\gets$ \textproc{Combine}($\Pi$)  \Comment{\textit{Combine log $p$-values $\Pi$ using Fisher's c.p.t.}}
    \State cnts $\gets$ $0^{|\mathbf{A}|}$ \Comment{\textit{Count vector of size equal to the number of alive variables}}
    \State cnts $\gets$ cnts + \textproc{I}($\pi \geq \alpha)$
    \For{b = 1 to B}
    \State $\Pi^b$ $\gets$ \textproc{BootstrapSample}$(\Pi)$
    \State $\pi^b$ $\gets$ \textproc{Combine}($\Pi^b$)  \Comment{\textit{Combine log $p$-values $\Pi^b $using Fisher's c.p.t.}}
    \State cnts $\gets$ cnts + \textproc{I}($\pi^b \geq \alpha)$
    \EndFor
    \State \Comment{\textit{Drop variables if $p$-value larger than $\alpha$ with probability at least $P_{drop}$}}
    \State $\mathbf{R} \gets \mathbf{R} \setminus \{X_i \in \mathbf{A} : cnts_i / (B+1) \geq P_{drop}\}$
    \State $\mathbf{A} \gets \mathbf{A} \setminus \{X_i \in \mathbf{A} : cnts_i / (B+1) \geq P_{drop}\}$
	\State {\bfseries return} $\langle \mathbf{R}, \mathbf{A} \rangle$
\end{algorithmic}
\end{algorithm}

\begin{algorithm}[t!]
\caption{\textproc{EarlyStopping}}
\label{alg:stop}
\begin{algorithmic}[1]
	\Require Log $p$-values $\Pi$, Alive Variables $\mathbf{A}$, Number of Bootstrap Samples $B$, ES Threshold $P_{stop}$
    \Ensure Alive Variables $\mathbf{A}$
    \State $\pi$ $\gets$ \textproc{Combine}($\Pi$)  \Comment{\textit{Combine log $p$-values $\Pi$ using Fisher's c.p.t.}}
    \State $X_{best}$ $\gets$ $\argmin\limits_{X_i \in \mathbf{A}}$ $\pi_i$ \Comment{\textit{Identify variable with minimum Fisher's combined $p$-value}}
    \State cnts $\gets$ $0^{|\mathbf{A}|}$ \Comment{\textit{Count vector of size equal to the number of alive variables}}
    \State cnts $\gets$ cnts + \textproc{I}($\pi_{best} < \pi$)
    \For{b = 1 to B}
    \State $\Pi^b$ $\gets$ \textproc{BootstrapSample}$(\Pi)$
    \State $\pi^b$ $\gets$ \textproc{Combine}($\Pi^b$)  \Comment{\textit{Combine log $p$-values $\Pi^b $using Fisher's c.p.t.}}
    \State cnts $\gets$ cnts + \textproc{I}($\pi^b_{best} < \pi^b$)
    \EndFor
    \State \Comment{\textit{Exclude variables from $\mathbf{A}$ that are worse than $V_{best}$ with probability at least $P_{stop}$}}
    \State $\mathbf{A} \gets \mathbf{A} \setminus \{X_i \in \mathbf{A} : cnts_i / (B+1) \geq P_{stop}\}$
	\State {\bfseries return} $\mathbf{A}$
\end{algorithmic}
\end{algorithm}

\begin{algorithm}[t!]
\caption{\textproc{EarlyReturn}}
\label{alg:return}
\begin{algorithmic}[1]
	\Require Log-likelihoods $\Lambda$, Alive Variables $\mathbf{A}$, Number of Bootstrap Samples $B$, ER Threshold $P_{return}$, ER Tolerance $lt$
    \Ensure Alive Variables $\mathbf{A}$
    \State $\pi$ $\gets$ \textproc{Combine}($\Pi$)  \Comment{\textit{Combine log $p$-values $\Pi$ using Fisher's c.p.t.}}
    \State $\lambda$ $\gets$ \textproc{SumRows}($\Lambda$)  \Comment{\textit{Sum rows of log-likelihoods $\Lambda$}}
    \State $X_{best}$ $\gets$ $\argmin\limits_{X_i \in \mathbf{A}}$ $\pi_i$ \Comment{\textit{Identify variable with minimum Fisher's combined $p$-value}}
    \State cnts $\gets$ $0^{|\mathbf{A}|}$ \Comment{\textit{Count vector of size equal to the number of alive variables}}
    \State cnt $\gets$ cnts + \textproc{I}$(\lambda_{best} - \lambda > lt)$
    \For{b = 1 to B}
    \State $\Lambda^b$ $\gets$ \textproc{BootstrapSample}$(\Lambda)$
    \State $\lambda^b$ $\gets$ \textproc{SumRows}($\Lambda$)  \Comment{\textit{Sum rows of log-likelihoods $\Lambda^b$}}
    \State cnts $\gets$ cnts + \textproc{I}$(\lambda^b_{best} - \lambda^b > lt)$
    \EndFor
    \State \Comment{\textit{Select $X_{best}$ early if better than all other variables with probability at least $P_{return}$}}
    \If{$\forall i, cnts_i / (B+1) \geq P_{return}$}
    \State $\mathbf{A}$ $\gets$ $\{X_{best}\}$
    \EndIf
	\State {\bfseries return} $\mathbf{A}$
\end{algorithmic}
\end{algorithm}

\subsection{Implementation Details of Bootstrap Testing}
\label{sec:algheuristics}
We recommend using the same sequence of bootstrap indices for each variable, and for each bootstrap test.
The main reasons are to (a) simplify implementation, (b) avoid mistakes and (c) ensure results do not change across different executions of the algorithms.
This can be done by initializing the random number generator with the same seed.
Next, note that ED, ES and ER do not necessarily have to be performed separately, but can be performed simultaneously (i.e,. using the same bootstrap samplings).
This allows the re-usage of the sampled indices for all tests and variables, saving some computational time.
Another important observation for ED and ES is that the actual combined $p$-values are not required.
It suffices to compare statistics instead, which are inversely related to $p$-values: larger statistics correspond to lower $p$-values.
For the ED test, the statistic has to be compared to the statistic corresponding to the significance level $\alpha$, which can be computed using the inverse $\chi^2$ cumulative distribution.
This is crucial to speed-up the procedure, as computing log $p$-values is computationally expensive.
Finally, note that the exact probabilities for the tests are not required, and one can often decide earlier if a probability is smaller than the threshold.
For example, let $P_{drop} = 0.99$ and $B = 999$.
Then, in order to drop a variable $X_i$, the number of times $cnts_i$ where the $p$-value of $X_i$ exceeds $\alpha$ has to be at least 990.
If after $K$ iterations $(B - K) + cnts_i$ is less than 990, one can determine that $X_i$ will not be dropped; even if in all remaining bootstrap iterations its $p$-value is larger than $\alpha$, $cnts_i + B - K$ will always be less than 990, and thus the probability $P(\pi_{i} \geq \alpha)$ will be less than the threshold $P_{drop} = 0.99$.

Finally, we note that, in order to minimize the probability of wrong decisions, large values for the ED, ES and ER thresholds should be used.
We found that values of $0.99$ for $P_{drop}$ and $P_{stop}$, and values of $P_{return} = 0.95$ and $tol = 0.9$ work well in practice.
Furthermore, the number of bootstraps $B$ should be as large as possible, with a minimum recommended value of 500.
By default, PFBP uses the above values.

\section{Tuning the Data Partitioning Parameters of the Algorithm}
\label{sec:parameters}
The main parameters for the data partitioning to determine are (a) the sample size $s$ of each Data Block, (b) the number of features $f$ in each Data Block, and (c) the number of Sample Subsets $C$ in each Group; the latter determines how many new $p$-values per feature are computed in each Group. 
Notice that {\em $s$ determines the horizontal partitioning of the data matrix and $f$ the vertical partitioning of data matrix}.
In general, the parameters are set so that Blocks are as small as possible to achieve high parallelization, without sacrificing feature selection accuracy: if the number of samples is too low, the local tests will have low power.
In this section, we provide detailed guidelines to determine those parameters, and show how those values were set for the special case of PFBP using conditional independence tests based on binary logistic regression.

\subsection{Determining the Required Sample Size for Conditional Independence Tests}
\label{sec:minsample}

For optimal computational performance, the number of Sample Sets should be as large as possible to increase parallelism, and each Sample Set should contain as few samples as possible to reduce the computational cost for performing the local conditional independence tests.
Of course, this should be done without sacrificing the accuracy of feature selection: if the number of samples is too low, the local tests will have low power.

Various rules of thumb have appeared in the literature to choose a sufficient number of samples for linear, logistic and Cox regression \cite{Peduzzi1996, RegressionModellingStrategies2001, Vittinghof2007}.
We focus on the case of binary logistic regression hereafter.
For binary logistic regression, it is recommended to use at least $s = c/\min(p_0,p_1) \cdot df$ samples, where $p_0$ and $p_1$ are the proportion of negative and positive classes in $T$ respectively, $df$ is the number of degrees of freedom in the model (that is, the total number of parameters, including the intercept term) and $c$ is usually recommended to be between 5 and 20, with larger values leading to more accurate results.
This rule is based on the events per variable (EPV) \cite{Peduzzi1996}, and will referred to as the EPV rule hereafter.

Rules like the above can be used to determine the number of samples $s$ in each Sample Set, by setting the minimum number of samples in each Data Block in a way that the locally computed $p$-values are valid for the type of test employed {\em in the worst case}. 
The worst case scenario occurs if the maximum number of features $\mathit{maxVars}$ have been selected.
If all features are continuous $df$ equals $\mathit{maxVars} + 1$.
This can easily be adapted for the case of categorical features, by considering the $\mathit{maxVars}$ variables with the most categories, and setting $df$ appropriately.
By considering the worst case scenario, the required number of samples can be computed by plugging the values of $df$, $c$, $p_0$ and $p_1$ into the EPV rule.
We found out that, although the EPV rule works reasonably well, it tends to overestimate the number of samples required for skewed class distributions.
As a result, it may unnecessarily slow down PFBP in such cases. 
Ideally for a given value of $c$ the results should be equally accurate irrespective of the class distribution and the number of model parameters.

To overcome the drawbacks of the EPV rule, we propose another rule, called the STD rule, which is computed as $s = df \cdot c/\sqrt{p_0 \cdot p_1}$.
For balanced class distributions the result is identical to the EPV rule, while for skewed distributions the value is always smaller.
We found that a value of $c = 10$ works sufficiently well, and recommend to always set $c$ to a minimum of 10; higher values could lead to more accurate results, but will also increase computation time.
Again, the number of samples per Sample Set is determined as described above.
A comparison of both rules is given in Appendix~\ref{app:minsample}.
We show that the STD rule behaves better across different values of $df$ and class distributions of the outcome than the EPV rule.

\subsection{Determining the Number of Features per Data Block}
Given the sample size $s$ of each data block $D_{i,j}$, the next hyper-parameter to decide is the number of features $f$ in each block. 
The physical partitioning to Feature Sets is performed so that each Block fits within the memory of a cluster node. 
Some \textit{physical partitioning is required only for ultra-high dimensional datasets and it was never launched in our implementation for the Sample Sets sizes of the datasets employed in our experiments}. 
In practice, features need to be partitioned only virtually, which is computationally cheaper. 
Specifically, enough (virtual) feature sets $\mathit{nf}$ are created so that the number of Data Blocks in a Group (i.e., $C \times \mathit{nf}$) is as close as possible to a desired oversubscription-of-machines parameter $o$.
The $o$ parameter dictates the average number of Blocks (tasks) assigned to a machine per Group. 
By default, the value of $o$ is set to 1.
In other words, we set $\mathit{nf} = \lfloor o \cdot M / C \rfloor$, where $M$ is the number of available machines, so that each machine processes at least $o$ blocks per Group.

\subsection{Setting the Number of Groups $C$}
We now discuss the determination of the $C$ value, the number of Sample Sets in each Group. 
Recall that, the value of $C$ determines how many Sample Sets are processed in parallel, (and thus, how many additional local $p$-values are added to the $p$-value matrix $\Pi$), before invoking bootstrap tests that decide on Early Dropping, Stopping or Return.
We set $C = 15$, as it allows enough local $p$-values for a bootstrap test to be performed in the first Group. 
Smaller values would invoke the bootstrap test more often and present more opportunities for Early Drop, Stop and Return, but would also reduce the parallelization of the algorithm (since computations await for the results of the bootstrap). 
The value 15 was chosen (without extensive experimentation) as a good trade-off between the two resources.

\section{Practical Considerations and Implementation Details}
\label{sec:pfbp:practical}

In this section, we discuss several important details for an efficient and accurate implementation of PFBP.
The main focus is on PFBP using conditional independence tests based on binary logistic regression, which is the test used in the experiments, although most details regard the general case or can be adapted to other conditional independence tests.

\subsection{Accurate Combination of Local $p$-values Using Fisher's method}
\label{sec:pfbp:practical:fisher}
In order to apply Fisher's combined probability test, {\em the data distributions of each data block should be the same for the test to be valid}. 
There should be no systematic bias on the data or the combining process may exacerbate this bias (see \cite{Tsamardinos2009}). 
Such bias may occur if blocks contain data from the same departments, stores, or branches, or in consecutive time moments and there is time-drift on the data distribution. 
This problem is easily avoided if before the analysis the partitioning of samples to blocks is done randomly, as done by PFBP.

Another important detail to observe in practice, is to \textit{directly compute the logarithm of the $p$-values for each conditional independence test} instead of first computing the $p$-value and then taking the logarithm.
As $p$-values tend to get smaller with larger sample sizes (in case the null hypothesis does not hold), they quickly reach the machine epsilon, and will be rounded to zero.
If this happens, then sorting and selecting features according to $p$-values breaks down and PFBP will select an arbitrary feature. 
This behavior is further magnified in case of combined $p$-values, as a single zero local $p$-value leads to a zero combined $p$-value no matter the values of the remaining $p$-values.

\subsection{Adapting the Number of Processed Groups to Improve Computational Efficiency}
In practice, we found that processing a single Group in each iteration of the $\textproc{ForwardIteration}$ algorithm may be slow in cases where no variables are dropped or stopped for multiple consecutive iterations.
In such cases, the algorithm will spend a large amount of time performing the bootstrap tests, even though in most cases no variables are removed from $\mathbf{R}$ or $\mathbf{A}$.
This can become especially problematic if the number of sample sets $\mathit{ns}$ is very large.
For this reason, we allow the algorithm to increase the Group size, if after processing two consecutive Groups the alive and remaining variables remain the same.
Specifically, we found that doubling the Group size works well in practice.
This is identical to doubling the number of Groups processed, and thus minimal changes are required in the algorithm.
One needs to keep track of the number of Groups to process in each iteration, and double that value if the alive and remaining variables do not change after the bootstrap tests.
Finally, we note that the value of $C$ is reset to the default value (15 in our case, see Section~\ref{sec:parameters}) after each forward Iteration.

\subsection{Implementation of the Conditional Independence Test using Logistic Regression for Binary Targets}
The conditional independence test is the basic building block of PFBP, and thus using a fast and robust implementation is essential. 
Next, we briefly review optimization algorithms used for maximum likelihood estimation, mainly focusing on binary logistic regression, and in the context of feature selection using likelihood-ratio tests.

A comprehensive introduction and comparison of algorithms for fitting (i.e., finding the $\beta$ that maximizes the likelihood) binary logistic regression models is provided in \cite{Minka2003}.
Three important classes of optimization algorithms are Newton's method, conjugate gradient descent and quasi-Newton methods.
Out of those, Newton's method is the most accurate and typically converges to the optimal solution in a few tens of iterations.
The main drawback is that each such iteration is slow, requiring $O(n \cdot d^2)$ computations, where $n$ is the sample size and $d$ the number of features.
Conjugate gradient descent and quasi-Newton methods on the other hand require $O(n \cdot d)$ and $O(n \cdot d + d^2)$ time per iteration, but may take much longer to converge.
Unfortunately, there are cases were those methods fail to converge to an optimal solution even after hundreds of iterations.
This not only affects the accuracy of feature selection, but also leads to unpredictable running times.
Most statistical packages include one or multiple implementations of logistic regression.
Such implementations typically use algorithms that can handle thousands of predictors, with quasi-Newton methods being a popular choice.
For feature selection however, one is typically interested to select a few tens or hundreds of variables.
In anecdotal experiments, we found that for this case Newton's method is usually faster and more accurate, especially with fewer than 100-200 variables.
Because of that, and because of the issues mentioned above, we used a fine-tuned, custom implementation of Newton's method.

There are some additional, important details.
First of all, there are cases where the Hessian is not invertible
\footnote{One case where this happens is if the covariance matrix of the input data is singular, or close to singular. Note that, due to the nature of the feature selection method which considers one variable at a time, this can happen only if the newly added variable is (almost) collinear with some of the previously selected variables. If this is the case, the variable would not be selected anyway.}.
If this the case, we switch to conjugate gradient descent using the fixed Hessian as a search direction for that iteration, as described in \cite{Minka2003}.
Finally, as a last resort, in case the fixed Hessian is not invertible we switch to simple gradient descent.
Next, for all optimization methods there are cases in which the computed step-size has to be adjusted to avoid divergence, whether it is due to violations of assumptions or numerical issues.
One way to do this is to use inexact line-search methods, such as backtracking-Armijo line search \cite{Armijo1966}, which was used in our implementation.

\subsection{Score Tests for the Univariate Case}
In the first step of forward selection where no variable has been selected, one can use a score test (also known as Lagrange multiplier test) instead of a likelihood-ratio test to quickly compute the $p$-value without having to actually fit logistic regression models.
The statistic of the Score test equals \cite{Hosmer2013}

$$ \text{Statistic} \equiv \frac{\sum_{j=1}^{n} X_j (T_j - \bar{T})}{\sqrt[]{\bar{T}(1 - \bar{T})\sum_{j=1}^{n} (X_j - \bar{X})^2}}$$
where $n$ is the number of samples, $T$ is the binary outcome variable (using a 0/1 encoding), and $X$ is the variable tested for independence.
Note that, such tests can also be derived for models other than binary logistic regression, but it is out of the scope of the paper.
The score test is asymptotically equivalent to the likelihood ratio test, and in anecdotal experiments we found that a few hundred samples are sufficient to get basically identical results, justifying its use in Big Data settings. 
Using this in place of the likelihood ratio test reduces the time of the univariate step significantly and is important for an efficient implementation, as the first step is usually the most computationally demanding one in the PFBP algorithm, as a large portion of the variables will be dropped by the Early Dropping heuristic.

\section{Optimality of PFBP on Distributions Faithful to Bayesian Networks and Maximal Ancestral Graphs}
\label{sec:theory}
Assuming an oracle of conditional independence, it can be shown that the standard Forward-Backward Selection algorithm is able to identify the optimal set of features for distributions faithful to Bayesian networks or maximal ancestral graphs \cite{Margaritis2000,Tsamardinos2003IAMB,Borboudakis2017}.
Unfortunately, the Early Dropping (ED) heuristic may compromise the optimality of the method.
ED may remove features that are necessary for optimal prediction of $T$. 
Intuitively, these features provide no predictive information for $T$ given $\mathbf{S}$ (are conditionally independent) but become conditionally \textit{dependent} given a superset of $\mathbf{S}$, i.e., after more features are selected. 
This problem can be overcome by using multiple Runs of the Forward-Backward Phases.
Recall that, each Run reinitializes the remaining variables with $\mathbf{R} = \mathbf{F} \setminus \mathbf{S}$.
Thus, each subsequent Run provides each feature with another opportunity to be selected, even if it was Dropped in a previous one.
The heuristic has a graphical interpretation in the context of probabilistic graphical models such as Bayesian networks and maximal ancestral graphs \cite{Pearl2000,Spirtes2000, SpirtesRichardson2002} inspired by modeling causal relations. 
A rigorous treatment of the Early Dropping heuristic and theorems regarding its optimality for distributions faithful to Bayesian networks and maximal ancestral graphs is provided in \citep{Borboudakis2017}; for the paper to be self-sustained, we provide the main theorems along with proofs next.

We assume that PFBP has access to an \textbf{independence oracle} that determines whether a given conditional dependence or independence holds.
Furthermore, we assume that the Markov and faithfulness conditions hold, which allow us to use the terms d-separated/m-separated and independent (dependent) interchangeably.
We will use the the \textbf{weak union} axiom, one of the \textbf{semi-graphoid} axioms \citep{Pearl2000} about conditional independence statements, which are general axioms holding in all probability distributions.
The weak union axiom states that $\condind{\mathbf{X}}{\mathbf{Y}\cup\mathbf{W}}{\mathbf{Z}} \Rightarrow \condind{\mathbf{X}}{\mathbf{Y}}{\mathbf{Z}\cup\mathbf{W}}$ holds for any such sets of variables.

\begin{theorem}\label{thm:ffbs1mb}
If the distribution can be faithfully represented by a Bayesian network, then PFBP with two runs identifies the Markov blanket of the target $T$.
\end{theorem}

\begin{proof}
In the first run of PFBP, all variables that are adjacent to $T$ (that is, its parents and children) will be selected, as none of them can be d-separated from $T$ by any set of variables.
In the next run, all variables connected through a collider path of length 2 (that is, the spouses of $T$) will become d-connected with $T$, since the algorithm conditions on all selected variables (including its children), and thus spouses will be selected as at least a $d$-connecting path is open: the path that goes through the selected child.
The resulting set of variables includes the Markov blanket of $T$, but may also include additional variables.
Next we show that all additional variables will be removed by the backward selection phase.
Let MB($T$) be the Markov blanket of $T$ and $\mathbf{S_{ind}} = \mathbf{S} \setminus $MB($T$) be all selected variables not in the Markov blanket of $T$.
By definition, $\condind{T}{\mathbf{X}}{MB(T)}$ holds for any set of variables $\mathbf{X}$ not in MB($T$), and thus also for variables $\mathbf{S_{ind}}$.
By applying the weak union graphoid axiom, one can infer that $\forall S_i \in \mathbf{S_{ind}}, \condind{T}{S_i}{MB(T) \cup \mathbf{S_{ind}} \setminus S_i}$ holds, and thus some variable $S_j$ will be removed in the first iteration.
Using the same reasoning and the definition of a Markov blanket, it can be shown that all variables in $\mathbf{S_{ind}}$ will be removed from MB($T$) at some iteration.
To conclude, it suffices to use the fact that variables in MB($T$) will not be removed by the backward selection, as they are not conditionally independent of $T$ given the remaining variables in MB($T$).
\qed
\end{proof}

\begin{theorem}\label{thm:ffbsinfmb}
If the distribution can be faithfully represented by a directed maximal ancestral graph, then PFBP with no limit on the number of runs identifies the Markov blanket of the target $T$.
\end{theorem}

\begin{proof}
In the first run of PFBP, all variables that are adjacent to $T$ (that is, its parents, children and variables connected with $T$ by a bi-directed edge) will be selected, as none of them can be m-separated from $T$ by any set of variables.
After each run additional variables may become admissible for selection.
Specifically, after $k$ runs all variables that are connected with $T$ by a collider path of length $k$ will become m-connected with $T$, and thus will be selected; we prove this next.
Assume that after $k$ runs all variables connected with $T$ by a collider path of length at most $k-1$ have been selected.
By conditioning on all selected variables, all variables with edges into some selected variable connected with $T$ by a collider path will become m-connected with $T$.
This is true because conditioning on a variable $Y$ in a collider $\langle X, Y, Z \rangle$ m-connects $X$ and $Z$.
By applying this on each variable on some collider path, it is easy to see that its end-points become m-connected.
Finally, after applying the backward selection phase, all variables that are not in the Markov blanket of $T$ will be removed; the proof is identical to the one used in the proof of Theorem~\ref{thm:ffbs1mb}.
\qed
\end{proof}

\section{Related Work}
\label{sec:related}
In this section we provide an overview of alternative parallel feature selection methods, focusing on methods for MapReduce-based systems (such as Spark), as well as causal-based methods, and compare them to PFBP.
An overview of feature selection methods can be found in \cite{Guyon2003} and \cite{Li2016Perspective}.

\subsection{Parallel Univariate Feature Selection and Parallel Forward-Backward Selection}
Univariate feature selection (UFS) applies only the first step of forward selection, ranks all features based on some ranking criterion, and selects either the top $\mathit{maxVars}$ variables or all features that satisfy some selection criterion.
An implementation for discrete data based on the chi-squared test of independence is provided in the Spark machine learning library MLlib \citep{Meng2016}.
In this case, all features are ranked based on the $p$-value of the test of unconditional independence with the outcome $T$, and features are selected by either choosing the top $\mathit{maxVars}$ ones, or all features with a $p$-value below a fixed significance level $\alpha$.
Although not explicitly mentioned as feature selection methods, MLlib also contains implementations of the Pearson and Spearman correlation coefficients, which can be used similarly to perform univariate feature selection with continuous features and outcome variables.
Furthermore, MLlib also contains implementations of binomial, multinomial and linear regression, which can be used both for univariate feature selection as well as for forward-backward selection (FBS), by performing likelihood-ratio tests.

The main advantages of PFBP over UFS and FBS are that (a) PFBP does not require specialized distributed implementations of independence tests, as it only relies on local computations and thus can use existing implementations, which is also much faster than fitting full models over all samples, and (b) it employs the Early Dropping, Early Stopping and Early Return heuristics, allowing it to scale both with number of features and samples. Perhaps, most importantly (c) UFS will not necessarily identify the Markov Blanket of $T$ even in faithful distributions; the solution by UFS will have false positives (e.g., redundant features) as well as false negatives (missed Markov Blanket features). 

\subsection{Single Feature Optimization}
The Single Feature Optimization algorithm (SFO) \citep{Singh2009} is a Mapreduce-based extension of the standard forward selection algorithm using binary logistic regression.
SFO (a) employs a heuristic that ranks the features at each step without the need to fit a full logistic regression model (that is, one over all samples) for all variables, and (b) uses a parallelization scheme to perform parallel computation over samples and features.
We note that, in contrast to PFBP, SFO does not require any specific data partitioning strategy.
We proceed by describing the ranking heuristic used by SFO.

Let $\mathbf{S}$ be the selected features up to some point and $\mathbf{R} = \mathbf{F} \setminus \mathbf{S}$ be all candidate variables for selection, and assume that a full logistic regression model $M$ for $T$ using $\mathbf{S}$ is available.
SFO creates an approximate model for each variable $R_i \in \mathbf{R}$ by fixing the coefficients of $\mathbf{S}$ using their coefficients in $M$, and only optimizing the coefficient of $R_i$.
This problem is much simpler than fitting full models for each remaining variable, significantly reducing running time.
Then, the best variable $R^*$ is chosen based on those approximate models (using some performance measure such as the log-likelihood), and a full logistic regression model $M^*$ with $\mathbf{S} \cup R^*$ is created.
Thus, at each iteration only a single, full logistic regression model needs to be created.
By default, SFO uses a maximum number of variables to select as a stopping criterion.
Alternatively, to decide whether $R^*$ should be selected a likelihood-ratio test could be used, in which case the test is performed on $M$ and $M^*$, and $R^*$ is selected if the $p$-value is below a threshold $\alpha$; we used this in our implementation of SFO in the experiments.
The parallelization over samples is performed in the map phase, in which one value $p_j$ is computed for each sample $j$, which equals 
$$p_j = \frac{e^{\beta \cdot \mathbf{S}_j}}{1 + e^{\beta \cdot \mathbf{S}_j}}$$\
where $\beta$ are the coefficients of $\mathbf{S}$ in $M$ and $\mathbf{S}_j$ are the values of $\mathbf{S}$ in the $j$-th sample.  
The values of $p_j$, the values of the outcome $T$ and all of candidate variables $\mathbf{R}$ are then sent to workers to be processed in the reduce phase.
Note that, this incurs a high communication cost, as essentially the whole dataset has to be sent over the network.
Finally, in the reduce phase, all workers fit in parallel over all variables $\mathbf{R}$ the approximate logistic regression models.

Although SFO significantly improves over the standard forward selection algorithm in running time, it has three main drawbacks compared to PFBP: (a) it has a high communication cost, in contrast to PFBP which only requires minimal information to be communicated, (b) to select a variable all non-selected variables have to be considered, while PFBP employs the Early Dropping heuristic that significantly reduces the number of remaining variables, and (c) SFO always uses all samples, while PFBP uses Early Stopping and Early Return allowing it to scale sub-linearly with number of samples. Finally, (d) SFO does not provide any theoretical guarantees.

\subsection{Information Theoretic Feature Selection for Big Data}
Information theoretic feature selection (ITFS) methods have been extended to Big Data settings \cite{Gallego17} and implemented for Spark\footnote{https://spark-packages.org/package/sramirez/spark-infotheoretic-feature-selection}.
They are applicable only for discrete variables, although discretization methods can be used to also handle continuous variables.
ITFS relies on computations of the mutual information and the conditional mutual information, and many variations have appeared in the literature \cite{Brown2012}; we provide a brief description next.
The criterion $J$ of many ITFS methods\footnote{There are methods that do not fall into this framework, but we will not go into more detail; see \cite{Brown2012} for more details.} for evaluating feature $X_k$ can be expressed as
$$
J(X_k) = I(T;X_k) - \beta \sum_{X_j \in \mathbf{S}} I(X_j;X_k) + \gamma \sum_{X_j \in \mathbf{S}} I(X_j;X_k|T)
$$
where $\beta$ and $\gamma$ are parameters taking values in $[0, 1]$.
The next best feature is chosen as the one maximizing $J$ with respect to the current set of selected variables $\mathbf{S}$.
All of those criteria approximate the conditional mutual information (CMI) $I(T;X_k|\mathbf{S})$ using different assumptions on the data distributions.
Observe that, both ITFS and forward selection are essentially identical, but use different criteria for selecting the next best feature.
Forward selection using regression models (e.g. logistic regression) assumes a specific probabilistic model to approximate the CMI.
Thus, both approaches use different types of approximations to the CMI.
We note that forward selection can also be used for discrete data to with the CMI criterion using the G-test of conditional independence \cite{Agresti2002}.

The main advantage of ITFS over PFBP is that computing (conditional) mutual informations does not require fitting any model, and thus can be performed efficiently even on a Big Data setting. They can additionally trivially take advantage of sparse data, further speeding up computation.
However, ITFS methods do not have the theoretical properties of PFBP, which can be shown to be optimal for distributions that can be faithfully represented by Bayesian networks and maximal ancestral graphs. 
This stems from the fact that PFBP solves an inherently harder problem, as it considers all selected variables at each iteration in order to select the next feature, while ITFS only conditions on one variable at a time.
Furthermore, ITFS methods are not as general as PFBP, which can be applied to various data types as long as an appropriate conditional independence test is available. For example, it is not clear if and how ITFS can be applied to time-to-event outcome variables, whereas PFBP can be directly applied if a likelihood-ratio test based on Cox regression is used.
Last but not least they are only applicable to discrete data. Thus, in case of continuous variables, a discretization method has to be applied before feature selection, possibly losing information \cite{Kerber1992, Dougherty1995}. This also increases computational time and may require extra tuning to find a good discretization of features.

\subsection{Feature Selection with Lasso}
Feature selection based on Lasso \cite{Tibshirani1996} implicitly performs feature selection while fitting a model.
The feature selection problem is expressed as a global optimization problem using an $L_1$ penalty on the feature coefficients.
We describe it in more detail next, focusing on likelihood-based models such as logistic regression.
Let $D(\theta)$ be the deviance of a model using $n$ parameters $\theta$.
The optimization problem Lasso solves can be expressed as
$$
\min_{\theta \in \mathbb{R}^n} D(\theta) + \lambda \norm{\theta}_{1}
$$
where $\norm{\theta}_{1}$ is the $L_1$ norm and $\lambda \geq 0$ is a regularization parameter.
The solutions Lasso returns are sparse, meaning that most coefficients are set to zero, thus implicitly performing feature selection.
The regularization parameter $\lambda$ controls the number of non-zero coefficients in the solution, with larger values leading to sparser solutions.
This problem formulation is a convex approximation of the more general best subset selection (BSS) problem \cite{Miller2002}, defined as follows to match the Lasso optimization formulation 
$$
\min_{\theta \in \mathbb{R}^n} D(\theta) + \lambda \norm{\theta}_{0}
$$
where $\norm{\theta}_{0}$ equals the total number of variables with non-zero coefficients.
The BSS problem has been shown to be NP-hard \cite{Welch1982}, and thus most approaches, such as Lasso and forward selection, rely on some type of approximation to solve it \footnote{Recently, there have been efforts for exact algorithms solving the BSS problem using mixed-integer optimization formulations for linear regression \cite{Bertsimas2016} and logistic regression \cite{Sato2016}.}.
In contrast to Lasso, which uses a different constraint on the values of the coefficients ($L_1$ instead of $L_0$ penalty), forward selection type algorithms perform a greedy optimization of the BSS problem, by including the next best variable at each step; see \cite{Miller2002,Friedman2001} for a more thorough treatment of the topic.
A sufficient condition for optimality of PFBP and FBS is that distributions can be faithfully represented by Bayesian networks or maximal ancestral graphs (see Section~\ref{sec:theory}.
Conditions for optimal feature selection with Lasso are given in \cite{Meinshausen2006}.

In extensive simulations it has been shown that causal-based feature selection methods are competitive with Lasso on classification and survival analysis tasks on many real datasets \cite{Aliferis2010JMLR,Lagani2010,Lagani2013,MXM16}.
Furthermore, the non-parallel version of PFBP (called Forward-Backward Selection with Early Dropping) as well as the standard Forward-Backward Selection algorithm have been shown to perform as well as Lasso if restricted to select the same number of variables \cite{Borboudakis2017}.

Lasso has been parallelized for single machines and shared-memory clusters \cite{Bradley2011,Asilomar2013,Li2016}.
These approaches only parallelize over features and not samples (i.e. consider vertical partitioning).
Naturally, ideas and techniques presented in those works could be adapted or extended for Spark or related systems.
An implementation of Lasso linear regression is provided in the Spark MLlib library \citep{Meng2016}.
A disadvantage of that implementation is that it requires extensive tuning of its hyper-parameters (like the regularization parameter $\lambda$ and several parameters of the optimization procedure), rendering it impractical as typically many different hyper-parameter combinations have to be used.
Unfortunately, we were not able to find any Spark implementation of Lasso for logistic regression, or any work dealing with the problem of efficient parallelization of Lasso on Spark.
Finally, we note that in contrast to forward selection using conditional independence tests, Lasso is not easily extensible for different problems, and requires specialized algorithms for different data types \citep{Meier2008, VanDeGeer2011, Ivanoff2016}, whose objective function may be non-convex \citep{VanDeGeer2011} or computationally demanding \citep{Fan2010}.

\subsection{Connections to Markov-Blanket Based Feature Selection}
Several algorithms have appeared in the literature that apply tests of conditional independence to select features. 
The theoretical properties of these algorithms often rely on the theory of Bayesian networks and the Markov blanket.
The GS \cite{Margaritis2000,Margaritis2009} and the IAMB \cite{Tsamardinos2003IAMB} algorithms, were some of the first to present the forward-backward selection algorithm in the context of Bayesian networks and the Markov blanket and prove correctness for faithful distributions. 
These algorithms perform tests of independence conditioning each time on {the full set $\mathbf{S}$ of selected features} and can guarantee to identify the Markov blanket for faithful distributions asymptotically. However, the larger the conditioning set, the more samples are required to obtain valid results. 
Thus, these algorithms are not well-suited for problems with large Markov blankets relative to the available sample size. 

Another class of such algorithms includes HITON-PC \cite{Aliferis2003HITON}, MMPC \cite{Tsamardinos2003MMPC}, and more recently SES \cite{MXM16} for multiple solutions. 
The main difference in this class of algorithms is that they condition on {\em subsets of the selected features $\mathbf{S}$}, not the full set. 
They do not guarantee to identify the full Markov blanket, but only a superset of the parents and children of $T$. 
Recent extensive experiments have concluded that they perform well in practice \cite{Aliferis2010JMLR}. 
These algorithms remove from consideration any features that become independent of $T$ conditioned on {\em some} subset of the selected features $\mathbf{S}$. 
This is similar to the Early Dropping heuristic and renders the algorithms quite computationally efficient and scalable to high-dimensional settings. 

PFBP combines the advantages of these two classes of algorithms: those that condition on subsets, drop features from consideration and achieve scalability, and those that condition on the full set of selected features and guarantee identification of the full Markov blanket. 

\section{Experimental Evaluation}
\label{sec:experiments}
We performed three sets of experiments to evaluate PFBP.
First, we investigate the scalability of PFBP in terms of variable size, sample size and number of workers on synthetic datasets, simulated from Bayesian networks.
Then, we compare PFBP to four competing forward-selection based feature selection algorithms. We made every reasonable effort to include all candidate competitors. These alternatives constitute algorithms specifically designed for MapReduce architectures (i.e., SFO), standard FS algorithms using parallel implementations of the conditional independence tests (i.e., UFS and FBS) and ITFS. The latter comparison is indirect using the published results in \cite{Gallego17}. The only Lasso implementation for Spark available in the Spark MLlib library \citep{Meng2016} (a) is for continuous targets, and thus is not suitable for binary classification tasks, and (b) required tuning of 5 hyper-parameters; as no procedure has been suggested by the authors for their tuning, it was excluded from the comparative evaluation. 
Finally, we performed a proof-of-concept experiment of PFBP on dense synthetic SNP data.

\subsection{Experimental Setup}
For all experiments we used a cluster with 5 machines, 1 acting as a master and 4 as workers, connected to a 1 Gigabit network.
Each machine has 2 Intel Xeon E5-2630 v3 CPUs with 8 physical cores each, and 256 GB of RAM.
Thus, a total of 64 cores and 1 TB of memory were available on all 4 workers.
The cluster is running Spark 2.1.0 and using the HDFS file system.
All algorithms were implemented in Java 1.8 and Scala 2.11.

The significance level $\alpha$ was set to 0.01 for all algorithms, and PFBP was executed with 2 Runs.
For the bootstrap tests used by PFBP, we used the default parameter values as described in Section~\ref{sec:algheuristics}.
To produce a predictive model for PFBP we followed the approach described in Section~\ref{sec:lrmodel}.
Parameter values related to the number of Group Samples, Sample Sets and Feature Sets were determined using the STD rule, and by setting the maximum number of variables to select to $\mathit{maxVars}$ (the exact value is given for each specific experiment later); see Section~\ref{sec:parameters} for more details.
We note that, none of the experiments required a physical partitioning to Feature Sets, and thus Feature Sets were only partitioned virtually.

\subsection{Scalability of PFBP with Sample Size, Feature Size and Number of Workers}

\begin{figure*}[t!]
\centering
\includegraphics[width=0.485\textwidth]{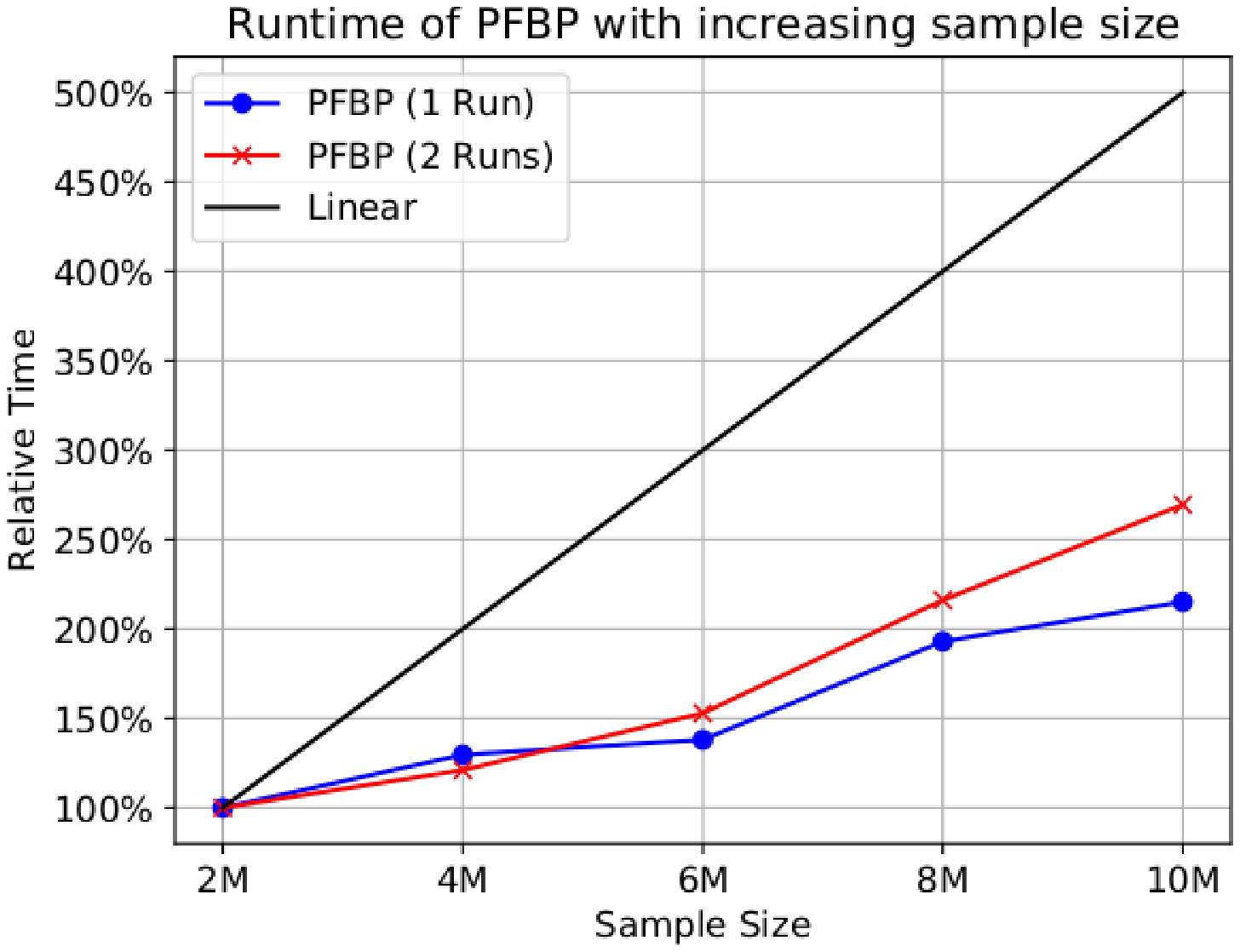}
\includegraphics[width=0.485\textwidth]{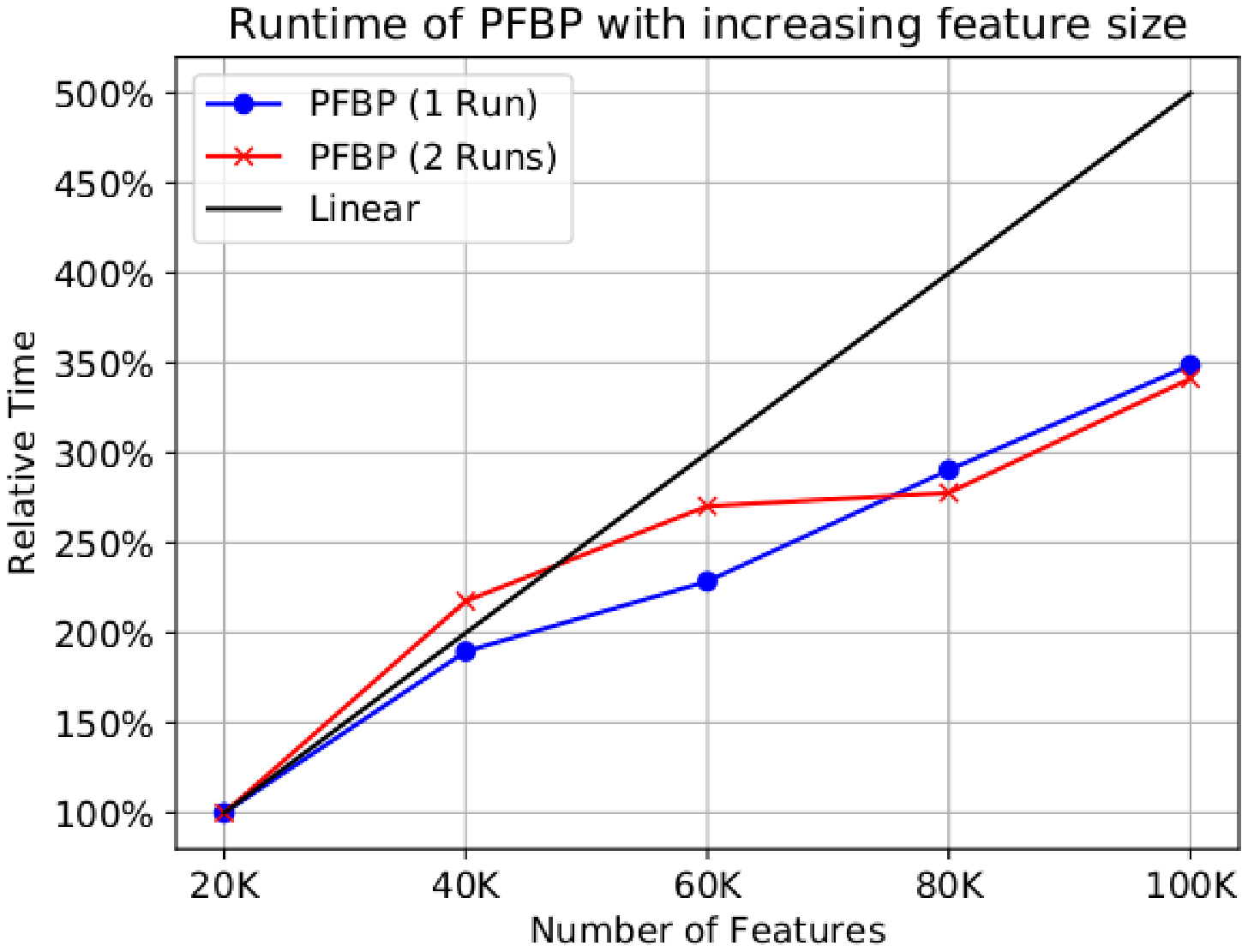}
\includegraphics[width=0.485\textwidth]{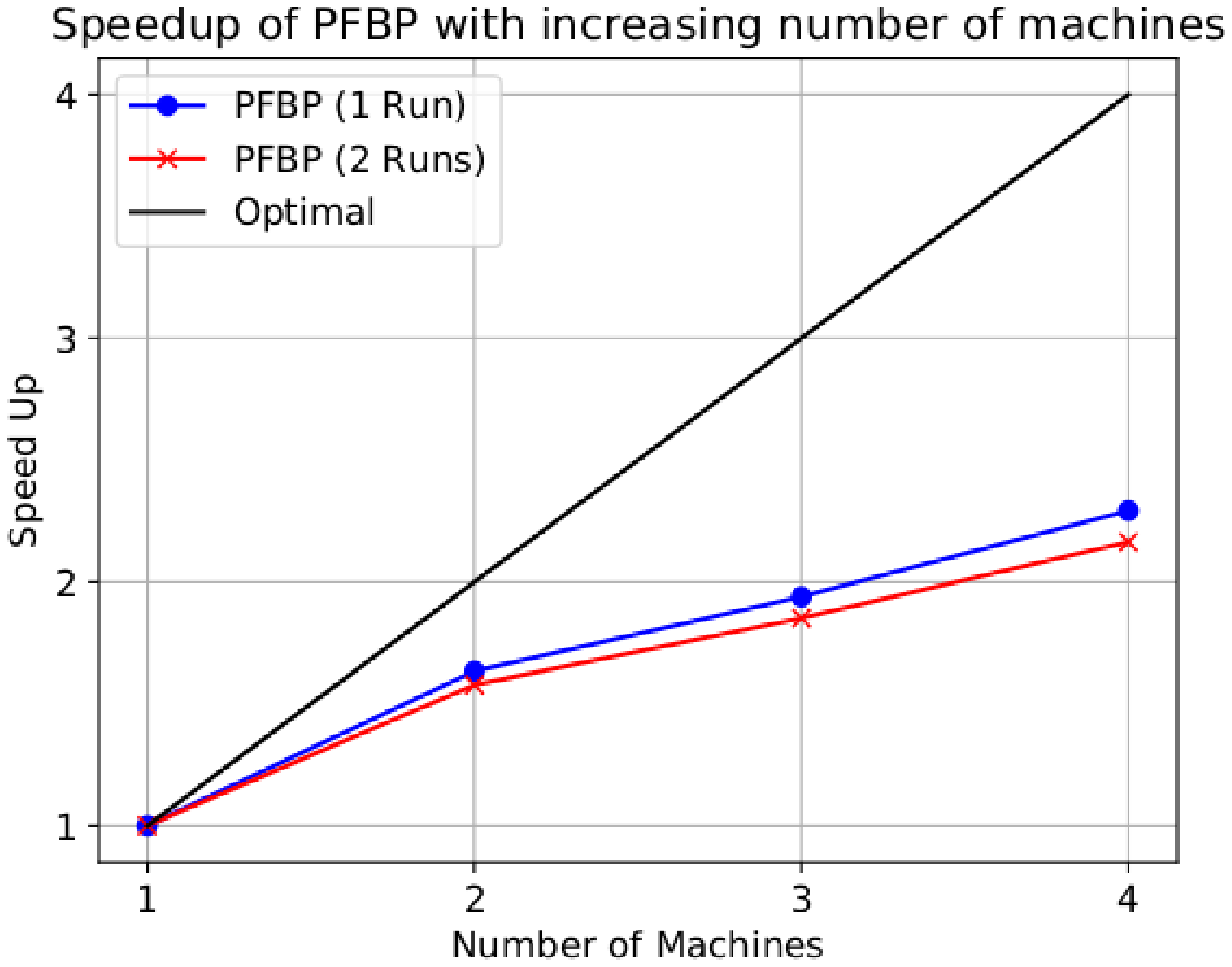}
\caption{Scalability of PFBP with increasing sample size (top left), feature size (top right) and number of machines (bottom). Time and speed-up were computed relatively to the first point on the x-axis, for the same number of Runs. PFBP improves super-linearly with sample size, linearly with feature size and running time is reduced linearly with increasing number of machines. The results are similar for PFBP with 1 run and 2 runs.}
\label{fig:scaling}
\end{figure*}

We investigated the scalability of PFBP on dense synthetic datasets in terms of sample size, variable size and number of workers used. 
The data were sampled from randomly generated Bayesian networks, which are probabilistic models that can encode complicated dependency structures among features.
Such \textit{simulated data contain not only features necessary for optimal prediction (strongly relevant in the terminology of \cite{John94}) and irrelevant, but also redundant features (weakly relevant \cite{John94})}. 
This is a novelty in the Big Data FS literature as far as we know, making the synthetic tasks more realistic.
A detailed description of the data and network generating procedures is given in Appendix~\ref{app:bn}.

For each experimental setting, we generated 5 Bayesian networks, and sampled one dataset from each.
The connectivity parameter $C$ was set to 10 (i.e., the average degree of each node), the class distribution of $T$ was 50/50, and the variance of the error term was set to 1.
To investigate scalability in terms of sample size, we fixed the number of features to 1000 and varied the sample size from 2M to 10M.
Scalability in terms of feature size was evaluated on datasets with 100K samples and varying the feature size from 20K to 100K, all of which also included the optimal feature set (i.e. the Markov blanket of $T$).
Finally, scalability in terms of number of workers was investigated on datasets with 10K variables and 1M samples.
The maximum number of variables $\mathit{maxVars}$ to select was set to 50.

The results are summarized in Figure~\ref{fig:scaling}.
On the top row we show the relative runtime of PFBP with varying sample size (left) and number of variables (right), respectively.
The bottom figure shows the speed-up achieved with varying the number of workers. 
Relative time and speed up are computed with respect to the lowest point on the x-axis. 
We can clearly see that: (Top Left) PFBP improves super-linearly with sample size; in other words, feeding twice the number of rows to the algorithm requires less than double the time. 
This characteristic can be attributed to the Early Stopping and Early Return heuristics.
(Top Right) PFBP scales linearly with number of features due to the Early Dropping heuristic and (Bottom) PFBP is able to utilize the allocated machines, although the speed-up factor does not reach the theoretical optimum. 
The reason for this is that the Early Stopping heuristic quickly prunes many features from consideration after processing the first Group sample, reducing parallelization in subsequent Groups as only few features remain Alive. 

\subsection{Comparative Evaluation of PFBP on Real Datasets}

\begin{table}[!t]
\centering
	\caption{Binary classification datasets used in the comparative evaluation}
    \label{tbl:datasets}
  \begin{tabular}{lrrr}
    \toprule
    Name &
    \#Samples &
    \#Variables &
    Non-zeros per Row
\\
     \midrule
     SUSY           & 5000000  & 18          & 17.78  \\
	 HIGGS          & 11000000 & 28          & 25.78  \\
     covtype.binary & 581012    & 54          & 12.94  \\
     epsilon        & 500000    & 2000       & 2000.00   \\
     rcv1.binary    & 697641    & 47236      & 73.15  \\
     avazu          & 40428967 & 1000000   & 14.99  \\
     news20.binary  & 19996     & 1355191   & 454.98 \\
     url            & 2396130  & 3231961   & 115.62 \\
     webspam        & 350000    & 16609143  & 3727.70 \\
     kdd2010a       & 8407752  & 20216830  & 36.34  \\
     kdd2010b       & 19264097 & 29890095  & 29.39  \\
    \bottomrule
  \end{tabular}
\end{table}

We evaluated the PFBP algorithm on 11 binary classification datasets, collected from the LIBSVM dataset repository\footnote{http://www.csie.ntu.edu.tw/~cjlin/libsvmtools/datasets/}, with the constraint that each dataset contains at least 500K samples or variables.
A summary of the datasets, shown in order of increasing variable size, can be found in Table~\ref{tbl:datasets}.
The first two columns show the total number of samples and variables of each dataset, while the last column shows the average number of non-zero elements of each sample.
The maximum number of non-zero elements equals the total number of variables.
Except for the first four datasets, all other datasets are extremely sparse.

\subsubsection{Algorithms and Setup}
We compared PFBP to 3 forward selection based algorithms: (i) Single Feature Optimization (SFO) \cite{Singh2009}, (ii) Forward-Backward Selection (FBS), and (iii) Univariate Feature Selection (UFS). 
UFS and FBS were implemented using a parallelized implementation of standard binary logistic regression for Big Data provided in the Spark machine learning library \cite{Meng2016}; this was also the implementation used to fit full logistic regression models for SFO at each iteration.

The algorithms were compared in terms of classification accuracy and running time. 
To estimate the classification accuracy, 5\% of the training instances were randomly selected and kept out.
The remaining 95\% were used by each algorithm to select a set of features and to train a logistic regression model using those features.
For SFO, FBS and UFS, the default recommended parameters for data partition and fitting the logistic regression models of Spark were used.
The maximum number of features to select was set to 50.
A timeout limit of 2 hours was used for each algorithm.
In case an algorithm did not terminate within the time limit, the number of features selected up to that point are reported.
If no feature was selected, the accuracy was set to N/A (not available).

\subsubsection{Results of the Comparative Evaluation}

\newcommand{\tmo}{$^1$120.0} 
\newcommand{\tmof}{$^2$120.0} 
\newcommand{\oom}{$^3$N/A} 
\newcommand{\crash}{$^4$N/A} 
\newcommand{\ito}{$^1$120.0} 
\newcommand{\restime}[4]{#1 & #2 & #3 & #4} 
\newcommand{\resacc}[5]{#1 & #2 & #3 & #4 & #5} 
\newcommand{\resvars}[4]{#1 & #2 & #3 & #4} 

\setlength{\tabcolsep}{.25em}
\begin{table*}[!t]
\centering
	\caption{
    The table shows the total running time in minutes, the classification accuracy \% and the number of variables selected.
    PFBP significantly outperforms all competitors in terms of running time, and is the only algorithm that is able to produce results on all datasets within the given time limit of 2 hours.
    SFO produces results only for datasets up to 2000 variables (epsilon), whereas UFS and FBS produce results only for datasets with a few tens of variables.
    Furthermore, for SFO, UFS and FBS Spark either crashed or ran out of memory for the largest datasets.
    In terms of classification accuracy, all algorithms perform similarly, and perform better than trivial classification to the most frequent class.
    }
    \label{tbl:resultsfull}
      \fontsize{7pt}{7pt}\selectfont
  \begin{tabular}{lrrrrccccccccc}
    \toprule
    &
    \multicolumn{4}{c}{Running Time (Minutes)} &
    \multicolumn{5}{c}{Classification Accuracy (\%)} &
    \multicolumn{4}{c}{Selected Variables}
\\
    \cmidrule(lr){2-5} \cmidrule(lr){6-10} \cmidrule(lr){11-14}
	Dataset &
    \multicolumn{1}{r}{PFBP} & \multicolumn{1}{r}{SFO} & \multicolumn{1}{r}{UFS} & \multicolumn{1}{r}{FBS} &
    \multicolumn{1}{c}{PFBP} & \multicolumn{1}{c}{SFO} &\multicolumn{1}{c}{UFS} & \multicolumn{1}{c}{FBS}  & \multicolumn{1}{c}{Trivial} &
    \multicolumn{1}{r}{PFBP} & \multicolumn{1}{r}{SFO} & \multicolumn{1}{r}{UFS} & \multicolumn{1}{r}{FBS}\\
    \midrule
    SUSY & \restime{0.1}{4.5}{1.0}{26.9} & \resacc{78.87}{78.55}{78.59}{78.59}{54.24} & \resvars{10}{14}{18}{15} \\
    HIGGS & \restime{0.2}{9.9}{3.0}{89.7} & \resacc{64.04}{64.07}{64.06}{64.06}{52.99} & \resvars{11}{18}{28}{18} \\
    covtype& \restime{3.0}{55.1}{1.6}{\tmo} & \resacc{75.79}{75.72}{75.70}{75.80}{51.24} & \resvars{27}{45}{50}{22} \\
    epsilon & \restime{1.3}{46.4}{\tmof}{\tmof} & \resacc{86.04}{86.39}{N/A}{N/A}{50.00} & \resvars{50}{50}{\color{red} 0}{\color{red} 0} \\
    rcv1 & \restime{10.8}{\tmof}{\tmof}{\tmof} & \resacc{91.32}{N/A}{N/A}{N/A}{52.47} & \resvars{50}{\color{red} 0}{\color{red} 0}{\color{red} 0} \\
    avazu & \restime{\tmo}{\oom}{\tmof}{\tmof} & \resacc{88.17}{N/A}{N/A}{N/A}{88.11} & \resvars{33}{\color{red} 0}{\color{red} 0}{\color{red} 0} \\
    news20 & \restime{55.2}{\tmof}{\tmof}{\tmof} & \resacc{85.84}{N/A}{N/A}{N/A}{50.00} & \resvars{50}{\color{red} 0}{\color{red} 0}{\color{red} 0} \\
    url & \restime{91.9}{\tmof}{\tmof}{\tmof} & \resacc{96.89}{N/A}{N/A}{N/A}{66.94} & \resvars{50}{\color{red} 0}{\color{red} 0}{\color{red} 0} \\
    webspam & \restime{\tmo}{\oom}{\oom}{\oom} & \resacc{97.60}{N/A}{N/A}{N/A}{60.62} & \resvars{33}{\color{red} 0}{\color{red} 0}{\color{red} 0} \\
    kdd2010a & \restime{\tmo}{\crash}{\crash}{\crash} & \resacc{86.13}{N/A}{N/A}{N/A}{85.30} & \resvars{27}{\color{red} 0}{\color{red} 0}{\color{red} 0} \\
    kdd2010b & \restime{\tmo}{\crash}{\crash}{\crash} & \resacc{86.13}{N/A}{N/A}{N/A}{86.06} & \resvars{25}{\color{red} 0}{\color{red} 0}{\color{red} 0} \\
    \bottomrule
    \multicolumn{14}{l}{$^1$ Timeout (returned model) $|$ $^2$  Timeout (no model returned) $|$ $^3$ Out of memory $|$ $^4$ Spark crashed}
  \end{tabular}
\end{table*}

Table~\ref{tbl:resultsfull} shows the results of the evaluation. 
It shows the running time in minutes, the classification accuracy and the number selected variables of each algorithm and on each dataset. 
We included the results of the trivial classification method, which classifies each sample to the most frequent class, and thus attains an accuracy equal to the frequency of the most common class.
It can clearly be seen that PFBP outperforms all competing methods in terms of running time.
All competing methods reach the timeout limit or crash and do not select a single feature even for the moderately sized rcv1 dataset, which contains only 47.2K variables and 697.6K samples.
PFBP reaches the timeout limit for 4 datasets, but is still able to select features (ranging from 25 to 33) and to produce a predictive model.
For the few datasets in which the competing methods terminated, all methods perform similarly in terms of accuracy.
For the avazu, kdd2010a and kdd2010b datasets, PFBP reaches an accuracy only marginally better than the trivial classifier, which may be because of the hardness of the problem, or because logistic regression is not appropriate.

\subsubsection{Indirect Comparison with Information Theoretic Feature Selection Methods}
We compare PFBP to the Minimum-Redundancy Maximum-Relevance (mRMR) algorithm \cite{Peng2005}, an instance of the ITFS methods \cite{Brown2012}.
The mRMR algorithm has been extended for Spark  and has been evaluated on binary classification datasets \cite{Gallego17}, including the epsilon, url and kdd2010b datasets.
This allows us to perform an indirect comparison of mRMR and PFBP.
They used a cluster with 432 cores in contrast to 64 used by us (6.75 times more cores), and used 80\% of the samples as training data in contrast to the 95\% we used.
mRMR was able to select 50 features in 4.9 and 5.6 minutes for epsilon and url respectively, and 12.9 minutes to select 25 features for kdd2010b. 
For epsilon, which is dense, PFBP took 1.3 minutes to select 50 features, while for url and kdd2010b which are sparse, PFBP needed 91.9 minutes and 120 minutes to select 50 and 25 features respectively.
Thus, for dense data PFBP may be faster, while possibly being a bit slower for sparse data, as mRMR had 6.75 times more cores available.
Recall that no specialized implementation for sparse data was used for PFBP, which could potentially further reduce its running time.
In any case, there does not seem to be a significant gap in running time between both approaches.
Unfortunately, no comparison in terms of classification accuracy is possible, as a different different performance measure and a different classifier was used for mRMR.
Finally, we point out that no predictive model was presented for the kdd2010b dataset, as the implementations of the naive Bayes and SVM classifiers used could not handle it, while PFBP is always able to produce a predictive model without any computational overhead, as long as it is able to complete a single iteration.

\subsection{Proof-of-Concept Application on SNP Data}
Single Nucleotide Polymorphisms (SNPs)\footnote{https://ghr.nlm.nih.gov/primer/genomicresearch/snp}, the most common type of genetic variation, are variations of a single nucleotide at specific loci in the genome of a species. 
The Single Nucleotide Polymorphism Database (dbSNP) (build 150) \cite{DBSNP} now lists 324 million different variants found in sequenced human genomes\footnote{https://www.ncbi.nlm.nih.gov/news/04-11-2017-human-snp-build-150/}. 
In several human studies, SNPs have been associated with genetic diseases or predisposition to disease or other phenotypic traits. 
As of 2017-08-13, the GWAS Catalog\footnote{https://www.ebi.ac.uk/gwas/} contains 3057 publications and 39366 unique SNP-trait associations. 
Large scale studies under way (e.g., Precision Medicine Initiative \cite{Collins2015}) intend to collect SNP data in large population cohorts, as well as other medical, clinical and lifestyle data. 
The resulting matrices may end up with millions of rows, one for each individual, and variables (SNP or some other measured quantity).
A proof-of-concept application of PFBP is presented next.

\subsubsection{SNP Data Generation and Setup}
We simulated genotype data containing 500000 individuals (samples) and 592555 SNP genotypes (variables), following the procedure described in \cite{Canela2015}.
As SNP data are dense, they require approximately 2.16 TB of memory, and thus are more challenging to analyze than sparse data, such as the ones used in the previous experiment.
The data were simulated with the HAPGEN 2 software \cite{HapGen} from the Hapmap 2 (release 22) CEU population \cite{HapMap}.
A more detailed description of the data generation procedure is given in Appendix~\ref{app:snp}.

We used $M = 100$ randomly selected SNPs to generate a binary phenotype (outcome), as described in \cite{Canela2015} (see also Appendix~\ref{app:snp}).
The optimal accuracy using all 100 SNPs is 81.42\%.
Ideally and given enough samples, any feature selection method should select those 100 SNPs and achieve an accuracy around 81.42\%.
Due to linkage disequilibrium however, many neighboring SNPs are highly correlated (collinear) and as a consequence offer similar predictive information about the outcome and are informationally equivalent.
Therefore, a high accuracy can be achieved even with SNPs other than the 100 used to simulated the outcome.

We used 95\% of the samples as a training set, and 5\% as a test set for performance estimation.
We set a timeout limit of 15 hours, and used the same setup as used in previous experiments, with the exception that the maximum number of variables to select was set to 100.

\subsubsection{Repartitioning to Reduce Memory Requirements}
For big, dense data such as the SNP data considered in this experiment which require over 2 TB of memory, a direct application of PFBP as used in other experiments is possible, but may be unnecessarily slow.
We found that for such problems, it makes sense to repartition the data at some point, if enough variables have been removed by the Early Dropping heuristic.
Repartitioning and discarding dropped variables reduces storage requirements, and may offer a significant speed boost.
It is an expensive operation however, and should only be used in special situations.
For the SNP data, after the first iteration only about a third of the variables remained, reducing the memory requirements to less than 1 TB, and thus most (if not all) of the data blocks were able to fit in memory.
In this case repartitioning makes sense, as its benefits far outweigh the computational overhead.

\subsubsection{Results on the SNP Data}
PFBP was able to select 84 features in 15 hours, using a total of 960 core hours.
It achieved an accuracy of 77.62\%, which is over 95\% of the theoretical optimal accuracy.
The results are very encouraging; in comparison the DISSECT software \cite{Canela2015} took 4 hours on 8400 cores (that is, 33600 core hours) and using 16 TB of memory to fit a mixed linear model on similar data, and to achieve an accuracy around 86\% of the theoretical maximum. The two experiments are not directly comparable because (a) the outcome in our case is binary instead of continuous requiring logistic regression models favoring DISSECT, (b) the scenarios simulated in \cite{Canela2015} had larger Markov blankets (1000 and 10000 instead of 100) favoring PFBP (although, their results are invariant to the size of the Markov blanket). Nevertheless, the reported results are still indicative of the efficiency of PFBP on SNP Big Data.

\begin{figure}[t!]
\centering
\includegraphics[width=0.485\columnwidth]{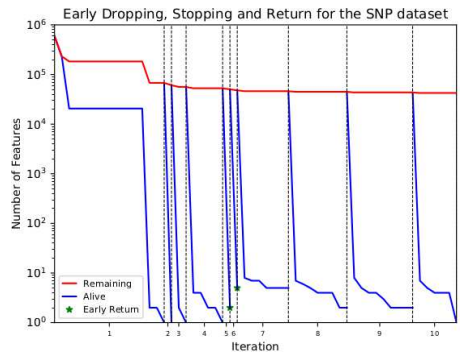}
\caption{The effects of the early pruning heuristics is shown for the first 10 forward iterations on the SNP data. The y-axis shows the number of variables on a logarithmic scale. The width of each iteration is proportional to the number of groups processed. The early dropping heuristic is able to quickly discard many features, reducing them by about an order of magnitude. Early stopping filters out most variables after processing the first group, and early return is applied two times.}
\label{fig:snp}
\end{figure}

Figure~\ref{fig:snp} shows the effects of the heuristics used by PFBP for the first 10 iterations.
The y-axis shows the number of Remaining and Alive features on a logarithmic scale.
The x-axis shows the current iteration, and the width is proportional to the total number of Groups processed in that iteration.
We observe that (a) Early Dropping discards many features in the first iteration, reducing the number of Remaining features by about an order of magnitude, (b) in most iterations, Early Stopping is able to reduce the number of Alive features to around 10 after processing the first Group, (c) Early Return is applied 2 times, ending the Iteration and selecting the top feature after processing a single Group.

\begin{figure}[t!]
\centering
\includegraphics[width=0.485\columnwidth]{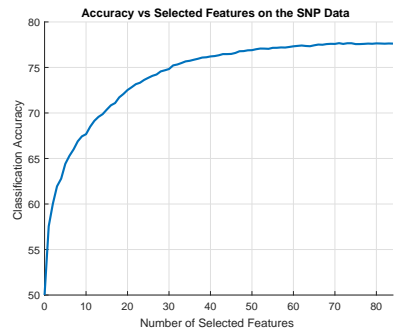}
\caption{The figure shows how the accuracy of PFBP on the SNP data increases as it selects more features. The models are produced by PFBP at each iteration with minimal computational overhead.
In the first few iteration, accuracy increases sharply, while in the later iterations a plateau is reached, reaching a value of 77.59\% with 70 features, with the maximum being 77.62\% with 84 features.
This could be used as a criterion to stop feature selection early.
}
\label{fig:snp:acc}
\end{figure}

Finally, we also computed the accuracy at each iteration of PFBP, to investigate its behavior with increasing number of selected features.
As before, the accuracy is computed on the 5\% of the data that were kept out as a test set.
Such information could be used to decide early whether a sufficient number of features has been selected, and to stop computation if the accuracy reaches a plateau.
This is often the case, as most important features are typically selected during the first few iterations.
This task can be performed using PFBP with minimal computational overhead, as the local models required to approximate a full global model (see Section~\ref{sec:lrmodel}) are already available.
The results are shown in Figure~\ref{fig:snp:acc}.
As expected, the largest increase in accuracy is obtained after selected the first few features, reaching an accuracy of 75\% even after selecting only 30 features.
In addition, after selecting about 70 features, the accuracy increases only marginally afterwards, increasing from 77.59\% with 70 features to 77.62\% with 84.
Thus, computation could be stopped after 70 features have been selected, and still attain almost the same accuracy.

\section{Discussion and Conclusions}
We have presented a novel algorithm for \emph{feature selection} (FS) in Big Data settings that can scale to millions of predictive quantities (i.e., features, variables) and millions of training instances (samples). The Parallel, Forward-Backward with Pruning (PFBP) enables \emph{computations that can be performed in a massively parallel way} by partitioning data both \emph{horizontally} (over samples) and \emph{vertically} (over features) and using meta-analysis techniques to combine results of local computations. It is equipped with \emph{heuristics that can quickly and safely drop from consideration some of the redundant and irrelevant features} to significantly speed up computations. The heuristics are inspired by causal models and provide theoretical guarantees of correctness in distributions faithful to causal models (Bayesian networks or maximal ancestral graphs). Bootstrapping testing allows PFBP to determine whether enough samples have been seen to safely apply the heuristics and forgo computations on the remaining samples. Our empirical analysis confirms that, PFBP exhibits a super-linear speedup with increasing sample size and a linear scalability with respect to the number of features and processing cores. 
A limitation to address in the future is to equip the algorithm with a principled criterion for the determining the number of selected features. Other directions to improve include exploiting the sparsity of the data, and implementing run-time re-partitioning when deemed beneficial, implementing tests in GPUs to name a few. 

\begin{acknowledgements}
IT and GB have received funding from the European Research Council under the European Union's Seventh Framework Programme (FP/2007-2013) / ERC Grant
Agreement n. 617393. We'd like to thank Kleanthi Lakiotaki for help with the motivating example text and data simulation, and Vincenzo Lagani for proofreading and constructive feedback.
\end{acknowledgements}

\clearpage
\appendix
\section*{Appendices}

\section{Accuracy of $p$-value Combination using Meta-Analysis and Evaluation of the STD Rule}
\label{app:minsample}

We evaluated the ability of the proposed $p$-value computation method (combination of local $p$-values using Fisher's combined probability test) in identifying the same variable to select as when global $p$-values are used.
We performed a computational experiment on simulated data to investigate the effect of the total sample size and number of data Blocks on the accuracy of the proposed approach.
Furthermore, we compare the STD and EPV rules for setting the minimum number of samples in each Data Block.
The EPV rule computes the sample size per Sample Group as $s = df \cdot c / \min(p_0,p_1)$, while STD uses $s = df \cdot c / \sqrt{p_0 \cdot p_1}$, where $df$ is set to the maximum number of degrees of freedom (see \ref{sec:minsample} for more details), $c$ is a positive constant (which may take different values for each rule), and $p_0$ and $p_1$ are the frequencies of the negative and positive class respectively.

\subsection{Data Generation}
To generate data with complex correlation structures, we chose to generate data from simulated Bayesian networks.
All variables are continuous Gaussian and are linear functions of their parents.
The target variable is binary, and the log-odds ratio is a linear function of its parents.
The procedure is described in detail in Appendix~\ref{app:bn}.

We used the following parameters to generate Bayesian networks and data from those networks: (a) the number of variables was set to 101 (100 variables plus the outcome $T$), (b) the connectivity parameter was set to 10 (i.e., the average degree of each node), (c) the frequency of the most frequent class of $T$ was set to $\{50\%, 60\%, 70\%, 80\%, 90\%\}$ and (d) the standard deviation of the error terms was set to $\{0.01, 0.1, 1\}$.
In total this results in 15 possible Bayesian network configurations.
Note that, the connectivity is relatively high and the standard deviation of the error terms is relatively low so that all variables are highly correlated, increasing the difficulty of the problem of selecting the best variable.
For each such parameter combination we generated 250 Bayesian networks, resulting in a total of 250 $\times$ 15 = 3750 networks.
Next, we generated datasets with different sample sizes, by varying the sample size from $10^{2.5}$ to $10^4$ in increments of $0.1$ of the exponent, leading to 16 different sample sizes.
Overall this resulted in 60000 datasets.

\subsection{Simulation Results: Combined $p$-values vs Global $p$-values}
\begin{figure*}[t!]
\centering
\includegraphics[width=0.475\textwidth]{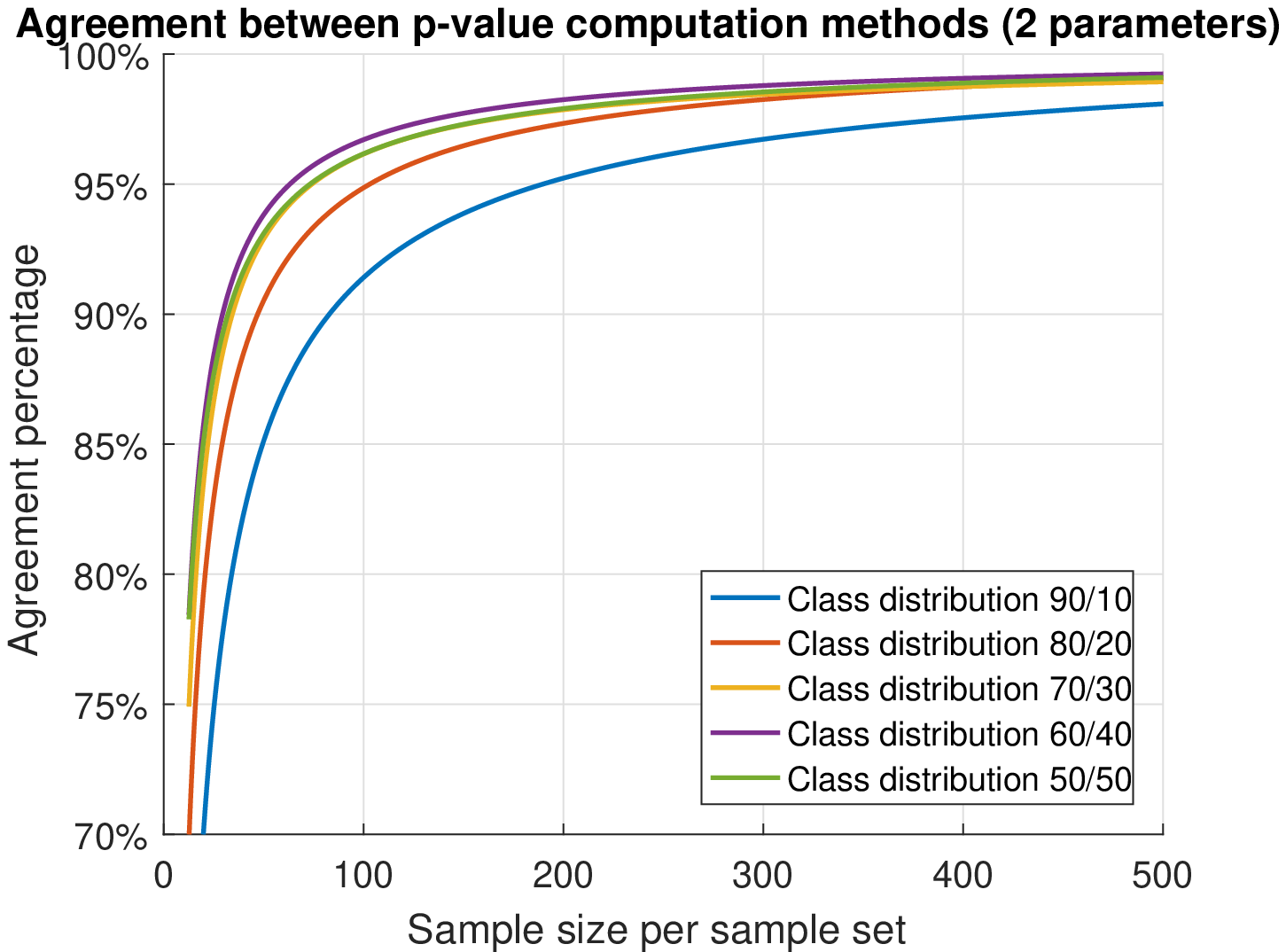}
\includegraphics[width=0.475\textwidth]{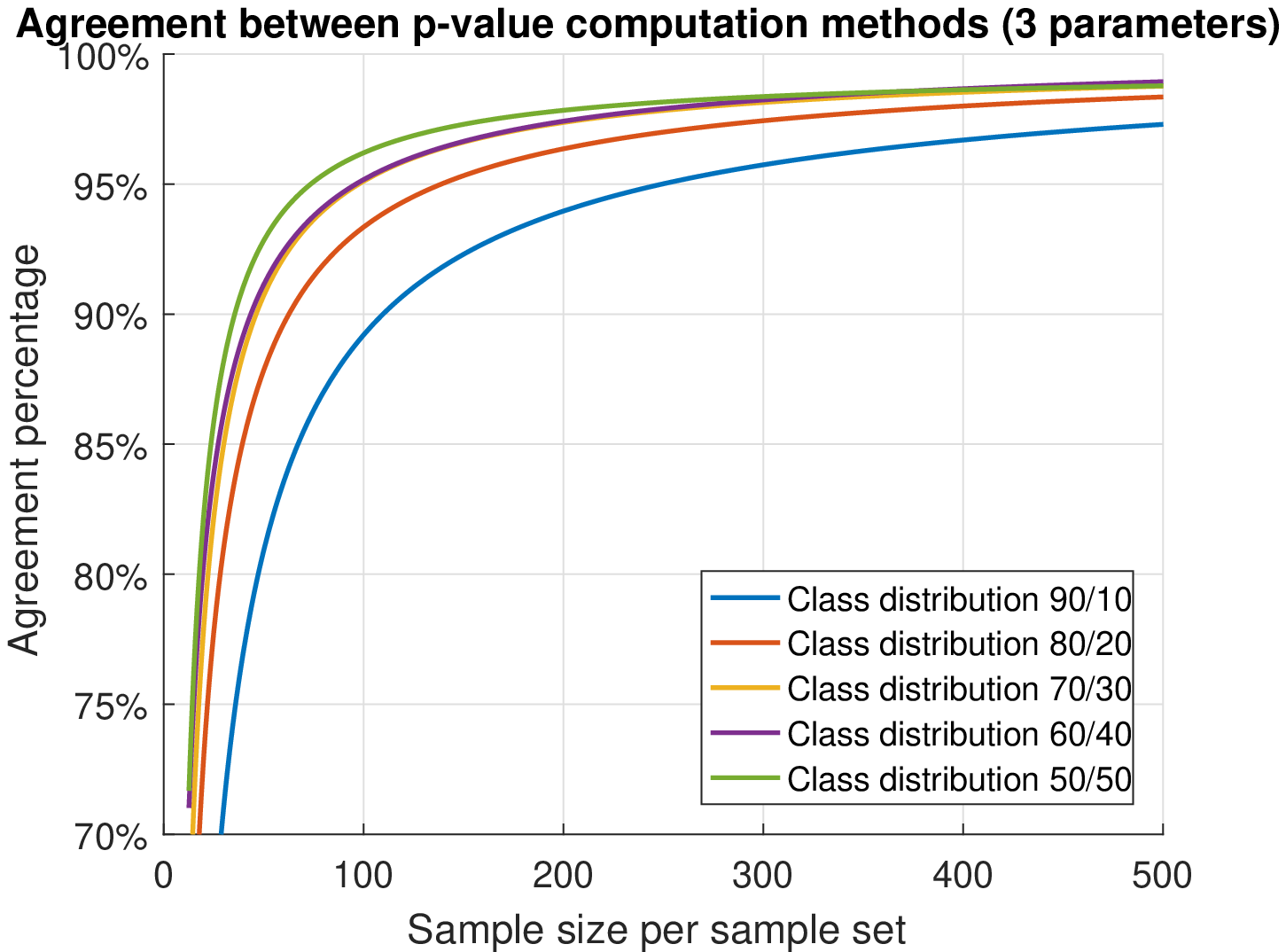}
\includegraphics[width=0.475\textwidth]{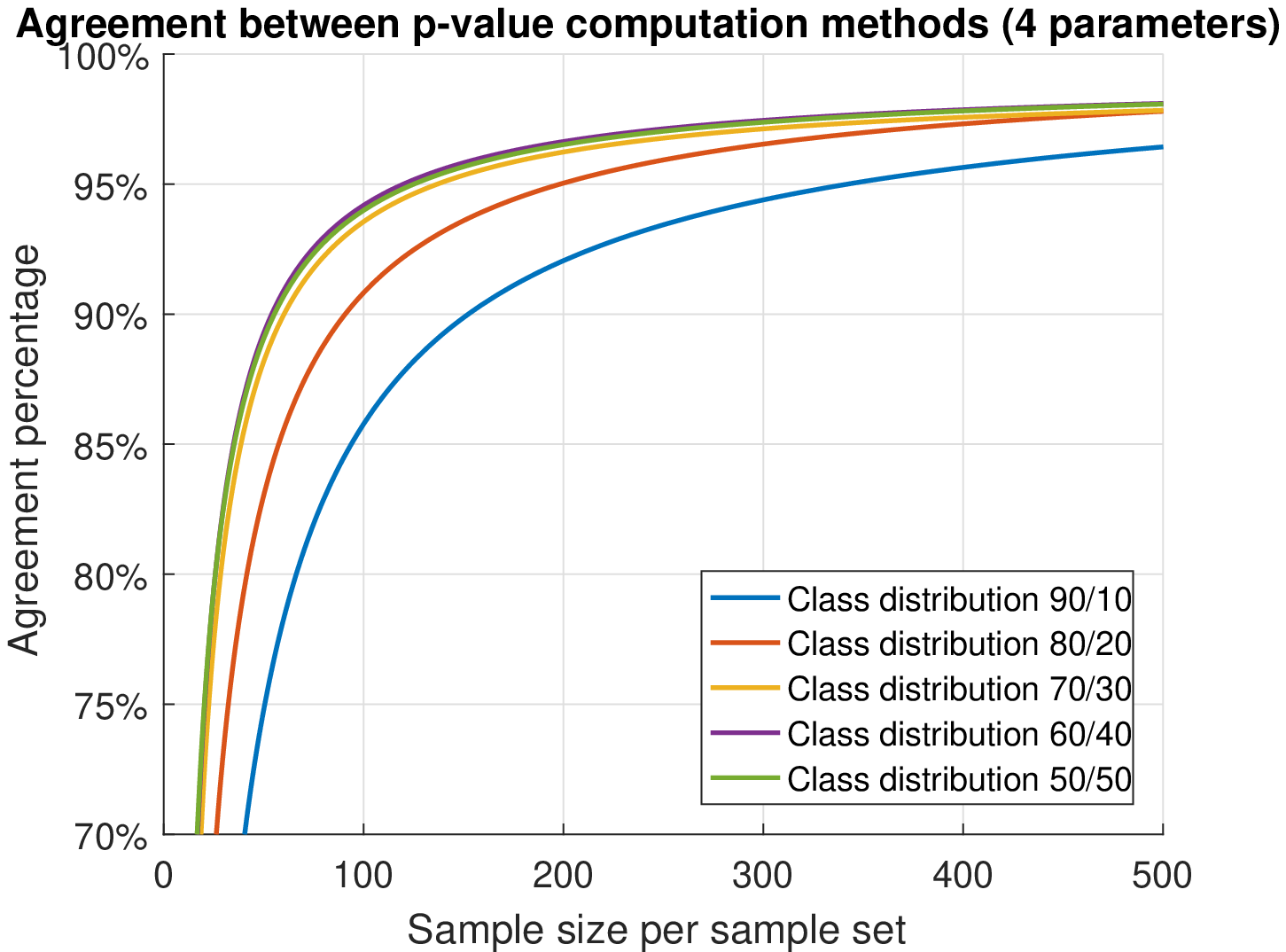}
\includegraphics[width=0.475\textwidth]{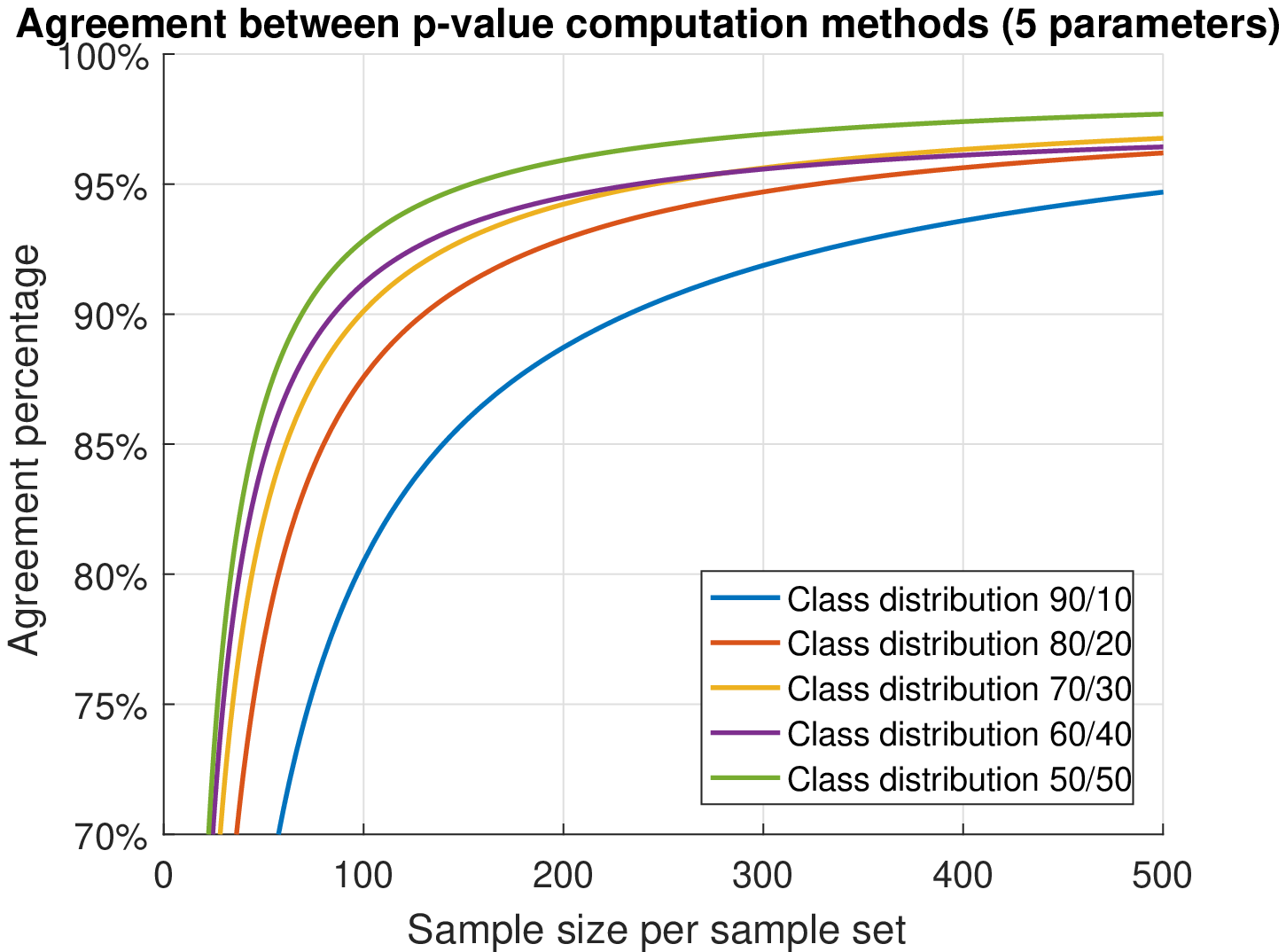}
\caption{The percentage of agreement is shown, which corresponds to how often combining local $p$-values and computing the $p$-value on all samples leads to the same decision. The $y$-axis shows how the sample size per sample set affects the agreement percentage. Both methods tend to agree asymptotically for various class distributions and conditioning set sizes.}
\label{fig:agreement}
\end{figure*}

We performed conditional independence tests on all generated datasets to simulate a forward Iteration using $p$-values from global tests (i.e., using all data) and from combined $p$-values using Fisher's method.
We varied the following parameters: (a) the number of conditioning variables, which was set to 0, 1, 2 or 3, and (b) the number of Sample Sets each dataset was split to, which ranged from 1 (no split) to 25 with increments of 1 (a total of 25 cases).
This allows us to investigate the effect of the total number of combined local $p$-values.
The simulation of a forward Iteration was performed for each dataset and conditioning size $k$ as follows: (a) $k$ variables were randomly selected from the Markov blanket of $T$ (simulating that $k$ variables have already been selected), (b) the global conditional independence test was performed between $T$ and the remaining variables over all samples (i.e. number of sample sets equals 1), (c) the same test was performed on all Sample Sets resulting by splitting the data randomly to $m$ equally-sized sample sets ($m$ ranging from 2 to 25) and combining the $p$-values using Fisher's combined probability test.

We compute the percentage of agreement between both methods, that is, how often both methods select the same variable.
This is computed as the proportion of times both methods agreed on the 250 repetitions, leading to one value for each of the 15 Bayesian network configurations, each sample size, conditioning set size and number of Sample Sets.
Thus, in total we have $15 \times 16 \times 4 \times 24$ = 23040 such values.
The results are summarized in Figure~\ref{fig:agreement}.
There are 4 figures, one for each different conditioning size, and each figure contains 5 curves, one for each class distribution of $T$.
Each such curve summarizes the results over all error variances, sample sizes and number of Sample Sets (that is, $3 \times 16 \times 24$ = 1152 points).
Note that the number of parameters of the largest model is always the conditioning size plus 2, as the model also includes the variable that is tested for (conditional) independence with $T$ and the intercept.
The $x$-axis shows the sample size per Sample Set, which is computed as the sample size divided by the number of Sample Sets.
We only show the results up to 500 samples per Sample Set; the agreement percentage approaches $100\%$ in all cases with increasing sample size, reaching at least $99\%$ with 5000 samples per Sample Set.
The $y$-axis shows how often both methods lead to the same decision.
To avoid cluttering, we computed the curves by fitting a power regression model $y = \alpha \cdot x^\beta + c$.
We found that this model is appropriate, as it has $R^2$ values between $0.75$ and $0.95$.
We conclude the following: \textit{(a) both approaches tend to make the same decision with increasing sample size, (b) the sample size per Sample Set required depends on the number of parameters and the class distribution, and increases with increasing number of parameters and class imbalance}.

\subsection{Simulation Results: STD vs EPV for Determining the Required Sample Size}
\begin{table}[!t]
\centering
	\caption{Median value of $c$ to obtain an agreement percentage between $85\%$ and $95\%$. $p_{max}$ corresponds to the proportion of the most frequent class, while df is the degrees of freedom in the largest model. The relative differences for the STD rule are smaller (less than 2 against over 2.5 for the EPV rule), suggesting it is more appropriate. A minimum value of $c = 10$ with the proposed rule is recommended and used in the experiments.}
    \label{tbl:essp}
  \begin{tabular}{lcrrrrrrrr}
    \toprule
    &
    \multicolumn{4}{c}{EPV Rule} &
    \multicolumn{4}{c}{STD Rule}
\\
    \cmidrule(lr){2-5} \cmidrule(lr){6-9} 
    \multicolumn{1}{c}{$p_{max}$ $|$ df} &
    \multicolumn{1}{c}{2} & \multicolumn{1}{c}{3} & \multicolumn{1}{c}{4} & \multicolumn{1}{c}{5} &
    \multicolumn{1}{c}{2} & \multicolumn{1}{c}{3} & \multicolumn{1}{c}{4} & \multicolumn{1}{c}{5} \\
     \midrule
     0.5 & 11.2 & 7.8 & 9.0 & 9.7 & 11.2 & 7.8 & 9.0 & 9.7 \\
     0.6 & 7.7 & 7.6 & 7.4 & 11.2 & 9.4 & 9.3 & 9.1 & 13.7 \\
     0.7 & 6.6 & 5.7 & 5.9 & 8.6 & 10.1 & 8.8 & 9.1 & 13.2 \\
     0.8 & 5.7 & 5.2 & 5.7 & 6.7 & 11.5 & 10.5 & 11.4 & 13.3\\
     0.9 & 5.0 & 4.4 & 4.2 & 4.4 & 14.9 & 13.2 & 12.5 & 13.3\\
    \bottomrule
  \end{tabular}
\end{table}

We propose an alternative rule to EPV, which is computed as $df \cdot c/\sqrt{p_0 \cdot p_1}$.
The denominator is the standard deviation of the class distribution, which follows a Bernoulli distribution.
We call this the STD rule hereafter.
For balanced class distributions the result is identical to the EPV rule, while for skewed distributions the value is always smaller.

To validate the STD rule, we used the results of the previous simulation experiment and computed the value of $c$ by solving the equation for $c$ and substituting in the values of the class distributions, degrees of freedom and sample size per Sample Set.
We kept the values of $c$ that correspond to an agreement percentage between $85\%$ and $95\%$ (focusing on an interesting range of high agreement between $p$-value computation methods), and computed their median value for each class distribution, conditioning size $k$ and for both rules.
Ideally, one would expect $c$ to be constant across rows (class distribution) and columns (conditioning size).
A constant value of $c$ for a rule means that the rule can exactly compute the required sample size to get an agreement percentage around $90\%$.
Furthermore, we note that the values of $c$ are not comparable between rules, and thus their exact values are not important; what matters is the relative difference between values of $c$ for the same rule.

The results are shown in Table~\ref{tbl:essp}.
Although the value of $c$ varies across class distributions and degrees of freedom, we can see that the relative differences are smaller the STD rule.
Specifically, for EPV $c$ ranges from 4.2 to 11.2, the latter being over 2.5 times larger, while for STD it ranges from 7.8 to 14.9, which is less than 2 times larger.
This suggests that the STD rule performs better than EPV across various conditioning set sizes and class distributions, at least on the experiments considered here.
Furthermore, the results suggest that a value of at least $c = 10$ should be used for STD to get reasonably accurate results.
We note that, in practice this value is much higher in most cases for PFBP, as it partitions the samples initially by considering the worst case scenario (i.e., selecting $\mathit{maxVars}$ variables).
Thus especially in early Iterations, which are the most crucial ones, PFBP will typically have a sufficient number of samples even with $c = 10$ to select the best variables.

\section{Simulating Data from Bayesian Networks}
\label{app:bn}
To generate data with complex correlation structures, we chose to generate data from Bayesian networks.
This is done in three steps: (a) generate a Bayesian network structure $\mathcal{G}$ with $N$ nodes (variables) and $M$ edges, (b) sample the parameters of $\mathcal{G}$, and (c) sample instances from $\mathcal{G}$.
We will next describe the procedures used for each step.

\subsection{Generating the Bayesian network structure}
First, we need to specify the number of nodes $N$ and the connectivity $C$ between those nodes, which implicitly corresponds to some number of edges $M$.
The connectivity parameter $C$ corresponds to the (average) degree of each variable.
Using the connectivity instead of setting the number of edges allows one to easily control the complexity of the network, as $C$ directly corresponds to the average number of parents and children of each node.
We proceed by showing how the edges were sampled.
Let $V_1, \dots, V_N$ be all nodes of $\mathcal{G}$, listed in topological order.
To sample the edges of the network we iterate over all pairs of variables $V_i$ and $V_j$ (i < j), and add an edge from $V_i$ to $V_j$ with probability $C/(N-1)$, ensuring acyclicity of the resulting graph.
It can be easily shown that this will result in a network with an average degree of $C$.

\subsection{Generating the Bayesian network parameters}
The first step is to pick the variable that corresponds to the outcome variable $T$.
We chose to use the node at position $\lceil N/2 \rceil$, as this node has the same number of parents and children on average.
For our experiments, we chose $T$ to be of binary type and all remaining variables to be of continuous type, but in principle everything stated can be easily adapted to other variable types.
In Bayesian networks, the each variable $V$ is a function of its parents $\mathbf{Pa}(V)$.
The functional form for continuous variables $V_i$ is
$$V_i = \beta_0 + \sum_{V_j \in \mathbf{Pa}(V_i)} \beta_j V_j + \epsilon_i$$
where $\beta_0$ is the intercept, $\beta_j$ is the coefficient of the $j$-th parent of $V_i$ and $\epsilon_i$ is its error term.
In our case, we set the intercept to 0, as it does not affect the correlation structure.
The coefficients $\beta_j$ are sampled uniformly at random from $[-1,-0.1] \cup [0.1,1]$ to avoid coefficients which are close to 0.
The error term $\epsilon_i$ follows a normal distribution with 0 mean and $\sigma_i^2$ variance, which was set to the default value of 1 in our experiments, unless stated otherwise.
Note that all variables are normally distributed as they are sums of normally distributed variables.
For each variable $V_i$ the mean equals zero and the variance equals $\sigma_i^2 + \sum_{V_j \in \mathbf{Pa}(V_i)} \beta_j^2$.
The fact that the variance increases may lead to numerical instabilities in practice, especially when generating large networks.
Because of that, we standardize each variable to have unit variance by dividing it with its standard deviation, which is the square root of the variance as described above.
For the target $T$, its log-odds ratio is again a linear function of its parents, defined as 
$$\log(\frac{P(T = 1)}{1 - P(T = 1)}) = \beta_0 + \sum_{V_j \in \mathbf{Pa}(T)} \beta_j V_j + \epsilon_T$$
As before, the log-odds ratio is standardized to have unit variance.

The value of $T$ is set to 1 whenever the log-odds ratio is larger than some threshold $t$, and to 0 otherwise.
Setting $t$ to 0 results in a 50/50 class distribution of $T$.
Other class distributions $p_0/p_1$ can be obtained by simply setting $t$ to $N_{0,1}^{-1}(1-p_0)$, where $N_{0,1}^{-1}$ is the standard normal inverse cumulative distribution function.
As a final note, the standardization method used above only guarantees that variables that come before $T$ in the topological ordering are standard normal variables.
As $T$ is not normally distributed (nor does it have unit variance), all variables that are direct or indirect functions of $T$ are not exactly normally distributed.
However, as this neither alters the correctness of the data generation method, nor leads to any other issues, we leave it as is.

\subsection{Sampling data from the generated Bayesian network}
To generate a sample, one has to traverse the network in topological order and to compute the value of each variable separately, using the formulas described previously.
By construction the network is already in topological order, which is simply given by the index of each variable.
To compute the value of a variable one has to compute the sum of its parents (if it has any parents), and to add the error term, which is drawn from a normal distribution.

\section{SNP Data Generation}
\label{app:snp}
To generate the SNP dataset we followed the procedure described in \cite{Canela2015}.
We used the HAPGEN 2 software \cite{HapGen} with the Hapmap 2 (release 22) CEU population \cite{HapMap} to simulate 500000 individuals (samples).
This population contains 2543887 SNPs, but only 592555 were kept, by filtering out the ones not available in the Illumina Human OmniExpress-12 v1.1 BeadChip \footnote{\url{https://support.illumina.com/array/array_kits/humanomniexpress-12-beadchip-kit.html}}.
The final dataset contains 500000 samples and 592555 SNPs.
Each variable takes values in $\{0,1,2\}$, which correspond to the number of reference alleles.
Thus, the dataset is dense, and requires approximately 2.16 TB memory (stored as double precision floats).
Naturally, fewer bytes can be used to store SNP data as each variable only takes 3 values, but this would require a specialized implementation.

\subsection{Phenotype Simulation}
Let $s_{ij}$ be the i-th value of the j-th SNP $s_j$, and $p_j$ be the reference allele frequency of SNP $j$, that is, $p_j$ is the average value of $s_{ij}$ divided by 2.
The standardized value of $s_{ij}$, $z_{ij}$ is defined as
$$z_{ij} = (s_{ij} - \mu_j) / \sigma_j$$
where $\mu_j = 2 p_j$ and $\sigma_j = \sqrt{2 p_j (1 - p_j)}$.

The phenotype (outcome) $y$ follows an additive genetic model
$$ y_i = g_i + e_i = \sum_{j=1}^M z_{ij} u_j + e_i$$
where $y_i$ is the i-th value of $y$, $g_i$ is the genetic effect, $e_i$ is the noise term, $M$ is the number of variables influencing $y$, and $u_j$ is the effect (coefficient) of $z_j$.
The coefficients $z_j$ were sampled from a normal distribution with zero mean and unit variance.
The error terms $e_i$ follow a normal distribution with zero mean and variance $\sigma_i^2 (1-h^2)/h^2$, where $\sigma_i^2$ is the variance of $g_i$ and $h^2$ corresponds to the trait heritability.
Naturally, the larger $h^2$, the more $y$ depends on the SNPs.
In our case, we chose $M = 100$ and set $h^2 = 0.7$, one of the values used in \cite{Canela2015}.
Finally, to obtain a binary outcome, we set the value of $y_i$ to $1$ if it is positive, and to $0$ otherwise, resulting in an approximately balanced outcome.

\bibliographystyle{abbrv}
\bibliography{ref}

\end{document}